\documentclass[sigconf,authorversion,nonacm]{aamas} 
\usepackage{array}
\usepackage{amsmath}
\usepackage[color={1 1 0},type=typewriter]{pdfcomment}
\usepackage{ stmaryrd }
\usepackage{balance}

\newcommand{\tobo}[1]{{\noindent \color{red} \textsc{tobo}: #1}}

\usepackage{tikz}
\usepackage{textcomp}
\usetikzlibrary{positioning,arrows,calc,fit,automata}

\newcommand{\s}{\scriptstyle}

\newcommand{\longv}[1]{} 

\setcopyright{ifaamas}
\acmConference{}
\copyrightyear{}
\acmPrice{}

\title[Attention! Dynamic Epistemic Logic Models of (In)attentive Agents]{Attention! \\Dynamic Epistemic Logic Models of (In)attentive Agents}

\author{Gaia Belardinelli}
\affiliation{
	\institution{University of Copenhagen}
	\city{Copenhagen}
	\country{Denmark}
}
\email{belardinelli@hum.ku.dk}

\author{Thomas Bolander}
\affiliation{
	\institution{Technical University of Denmark}
	\city{Kgs.\ Lyngby}
	\country{Denmark}}
\email{tobo@dtu.dk}

\begin{abstract}
	Attention is the crucial cognitive ability that limits and selects what information we observe. Previous work by Bolander et al. (2016) proposes a model of attention based on dynamic epistemic logic (DEL)  where agents are either fully attentive  or not attentive at all. While introducing the realistic feature that inattentive agents believe nothing happens, the model does not represent the most essential aspect of attention: its selectivity. Here, we propose a generalization that allows for paying attention to subsets of atomic formulas. We introduce the corresponding logic for propositional attention, and show its axiomatization to be sound and complete. We then extend the framework to account for inattentive agents that, instead of assuming nothing happens, may default to a specific truth-value of what they failed to attend to (a sort of prior concerning the unattended atoms). This feature allows for a more cognitively plausible representation of the inattentional blindness phenomenon, where agents end up with false beliefs due to their failure to attend to conspicuous but unexpected events. 
	Both versions of the model define attention-based learning through appropriate DEL event models based on a few and clear edge principles. While the size of such event models grow exponentially both with the number of agents and the number of atoms, 
	we introduce a new logical language for describing event models syntactically 
	and show that using this language our event models can be represented linearly in the number of agents and atoms. Furthermore, representing our event models using this language is achieved by a straightforward formalisation of the aforementioned edge principles.  
\end{abstract}

\keywords{Dynamic Epistemic Logic; Attention; Inattentional Blindness; Default Values; Syntactic Event Models; Succinctness}

\newcommand{\BibTeX}{\rm B\kern-.05em{\sc i\kern-.025em b}\kern-.08em\TeX}

\begin{document}
	\emergencystretch 3em

	\pagestyle{fancy}
	\fancyhead{}
	
	\settopmatter{printfolios=true}
	\maketitle


	\section{Introduction}
	Attention is the capacity of the mind to focus on a specific subset of available information. It limits and selects what we observe, to the extent that we may only consciously perceive events that receive our focused attention \cite{simons1999gorillas}. A fascinating family of phenomena suggesting that attention is necessary for visual awareness is the one where agents completely miss conspicuous events even when they happen at fixation. 
	Inattentional blindness is one such phenomenon \cite{mack1998inattention}. The name is suggestive of a form of cognitive blindness to external stimuli, which has been robustly replicated in the cognitive science literature. A famous experiment to test it is the so called \emph{Invisible Gorilla} video~\cite{gorilla_youtube}, by Simons and Chabris \cite{simons1999gorillas}. It is an online video where subjects are asked to ``Count how many times the players wearing white pass the basketball". While they focus on counting the ball passages, a clearly visible person in a gorilla costume crosses the scene. It is an unexpected appearance in such a situation, but it is an appearance right at fixation, as the gorilla passes in the middle of the group of players. Yet, Simons and Chabris found that about 50\% of subjects do not perceive any gorilla \cite{simons1999gorillas, gorilla_youtube}. Interestingly, subjects are often surprised when they realise to have missed such a salient event. This surprise has been taken to reveal a metacognitive error about the completeness of visual awareness, or in other words, an incorrect belief about attention capacities \cite{simons2011believe, simons2012common}. Indeed, researchers have shown that it is common for people to believe that they would notice much more than they in fact do, and that when they fail to attend something, they are not only uncertain about what they missed, but they often believe that what they did not notice, did not happen \cite{simons1999gorillas}. Attention and its limitations thus have substantial implications in people's belief dynamics, both in the sense that they severely limit what information is received, and in the sense that subjects often hold definite beliefs about events that they attend or fail to attend.

	Dynamic Epistemic Logic (DEL) is a branch of epistemic logic that has been used to study the dynamics of knowledge and beliefs \cite{ditmarsch2007dynamic}. Only relatively recently has there been investigations into the notion of attention and related phenomena in the DEL literature. For example, Bolander et al. (2016) introduce a form of attention in a DEL framework, representing it as an atomic formula $\mathsf{h}_a$ that, if true, expresses that agent $a$ pays attention to everything happening (any formula announced or any fact revealed) and, if false, that $a$ pays attention to nothing at all \cite{bolander2015announcements}. That work might be considered as a first step towards modelling the rich and complex phenomenon of attention in DEL, and the present work may then be considered a second step. 
	One of our contributions consists in generalizing the framework from  \cite{bolander2015announcements} so that agents can pay attention to any subset of atomic formulas. We encode attention by means of attention atoms $\mathsf{h}_ap$, for each agent $a$ and proposition $p$. 
	For the dynamic part, we generalise the event models from \cite{bolander2015announcements} by first recasting them using a few and clear edge principles. 
	Then, we gradually introduce different event models by building on this version of their model. What we do is the following: First, we account for agents that may have false beliefs about their attention (as in the inattentional blindness phenomenon above). 
	Second, we account for the dynamics of partial learning happening when agents only focus on a subset of the occurring events. In this version of the model, agents learn the part of the events that they are paying attention to, but keep intact their beliefs about what they did not attend to. As we have seen above, in reality, this is not always the case. In inattentional blindness for example, it often happens that inattentive agents change their beliefs to specifically account for the assumption that unattended events \emph{did not occur}. Then, as a third step, we add default values as a parameter of event models, which are a sort of prior that agents have and use to update their beliefs in case they miss some information. This addition gives us a more cognitively plausible representation of the experimental findings mentioned above, as now agents can default to the non-existence of the gorilla in the video even if they were previously uncertain about it. We introduce a logic for the first model of propositional attention (without defaults), and prove its axiomatization sound and complete. Lastly, we show that our idea of representing edges of event models by edge principles can be generalised to a new type of syntactic event models where events and edges are specified using logical formulas. We show exponential succinctness of these syntactic event models as compared to standard (semantic) event models. 
	
	Besides providing insights into how human attention interacts with beliefs, this research also goes towards the improvement of human-AI interaction, as it may help e.g.\ robots to reason about humans, required in human-robot collaboration settings. As explained by Verbrugge \cite{Verbrugge2009matter}, it's potentially dangerous if a robot in a human-robot rescue team makes too optimistic assumptions about the reasoning powers of human team members. The robot might for example falsely rely on a human to have paid attention to a certain danger, where in fact the human didn't. A proactively helpful robot should be able to take the perspective of the human and reason about what the human might or might not have paid attention to, and therefore which false beliefs the human might have. This requires that the robot has a model of the attention system of the human, and how this impacts her beliefs. We believe our models can be used in this way. Concretely, there has already been research on using epistemic planning based on DEL for human-robot collaboration~\cite{bolander2021del}, and since the models of this paper are also based on DEL, they lend themselves to immediate integration into such frameworks and systems. 

	This paper is an extended version of our paper accepted for AAMAS 2023 (paper \#1142). It has been extended with the proofs from the supplementary material of the original submission. 
	\section{Propositional attention}
	\subsection{Language}
	Throughout the paper, we use $Ag$ to denote a finite set of \emph{agents}, $At$ to denote a finite set of \emph{propositional atoms}, and we 
	let $H=\{\mathsf{h}_ap\colon p\in At, a\in Ag\}$ denote the corresponding set of \emph{attention atoms}. 
	With $p\in At, a\in Ag, \mathsf{h}_ap\in H$ and $\mathcal{E}$ being a multi-pointed event model\footnote{Defined further below. As usual in DEL, the syntax and semantics are defined by mutual recursion~\cite{ditmarsch2007dynamic}.}, define the language $\mathcal{L}$ by:\footnote{So $\mathcal{L}$ takes the sets $Ag$ and $At$ as parameters, but we'll keep that dependency implicit throughout the paper.} 
	\[
	\varphi::=\top\mid p\mid \mathsf{h}_a p\mid\neg\varphi\mid\varphi\wedge\varphi\mid B_a\varphi \mid [\mathcal{E}]\varphi.
	\]
	The attention atom $\mathsf{h}_ap$ reads ``agent $a$ is paying attention to whether $p$'',\footnote{The $\mathsf{h}$ in the attention formula stands for $\mathsf{h}$earing. It was proposed in \cite{bolander2015announcements}, and we keep it as we take their framework as our starting point.} $B_a\varphi$ reads ``agent $a$  believes $\varphi$'', and the dynamic modality $[\mathcal{E}]\varphi$ reads ``after $\mathcal{E}$ happens, $\varphi$ is the case". The formulas in $At \cup H \cup \{ \top \}$ are called the \emph{atoms}, and a \emph{literal}  is an atom or its negation. We often write $\bigwedge S$ to denote the conjunction of a set of formulas $S$. If $S$ is empty, we take $\bigwedge S$ as a shorthand for $\top$. 
	To keep things simple, we will assume that all consistent conjunction of literals are in a normal form where: (i) each atom occurs at most once; (ii) $\top$ doesn't occur as a conjunct, unless the formula itself is just $\top$; and (iii) the literals occur in a predetermined order (ordered according to some total order on $At \cup H$). This implies that given any disjoint sets of atoms $P^+$ and $P^-$, there exists a unique conjunction of literals (in normal form) containing all the atoms of $P^+$ positively and all the atoms of $P^-$ negatively. \longv{If $P^+ = P^- = \emptyset$, we have the empty conjunction, which as stated above is just taken to be shorthand for $\top$.}
	For conjuncts that are \emph{not} on this normal form, we assume them to always be replaced by their corresponding normal form. \longv{By this convention, we for instance have $\{ p \land q \land \top, q \land p \} = \{ p \land q \}$, since the two formulas would reduce to the same normal form (think of the conjunctive form as a shorthand for the aforementioned specification $P^+,P^-$ of the positive and negative literals contained in the conjunction).}   
	For any conjunction of literals $\varphi = \bigwedge_{1 \leq i \leq n} \ell_i$ and any literal $\ell$, we say that $\varphi$ \emph{contains} $\ell$ if $\ell = \ell_i$ for some $i$, and in that case we often write $\ell \in \varphi$. 
	For any conjunctions of literals $\varphi$, we define $\mathit{Lit}(\varphi)$ to be the set of literals it contains, that is, $\mathit{Lit}(\varphi) = \{ \ell \mid \ell \in \varphi \}$. For an arbitrary formula $\varphi$, we let $At(\varphi)$ denote the set of propositional atoms appearing in it. 
	\subsection{Kripke Model and Dynamics}
	We are going to model attention and beliefs using 
	DEL~\cite{ditmarsch2007dynamic}, where static beliefs are modelled by pointed Kripke models, and attention-based belief updates are modelled by multi-pointed event models (our product update and satisfaction definitions will be  slightly non-standard due to the multi-pointedness of the event models). 
	\begin{definition}[Kripke Model]\label{def: kripke model} A \emph{Kripke model} is a tuple $\mathcal{M}=(W,R,V)$ where $W\not=\emptyset$ is a finite set of \emph{worlds}, $R:Ag\rightarrow \mathcal{P}(W^{2})$ assigns an \emph{accessibility relation} $R_a$ to each agent $a\in Ag$, and $V:W\rightarrow\mathcal{P}(At\cup H)$ is a \emph{valuation function}. Where $w$ is the \emph{designated world}, we call $(\mathcal{M},w)$ a \emph{pointed Kripke model}.
	\end{definition}
	\begin{definition}[Event Model]
		\label{def: event model} An \emph{event
			model} is a tuple $\mathcal{E}=(E,Q,pre)$
		where $E\neq\emptyset$ is a finite set of \emph{events}, $Q:Ag\rightarrow \mathcal{P}(E^{2})$ assigns an \emph{accessibility relation} $Q_a$ to each agent $a\in Ag$ and $pre:E\rightarrow\mathcal{L}$ 
		assigns a \emph{precondition} to each event $e\in E$. 
		Where $E_d\subseteq E$ is a set of \emph{designated events}, $(\mathcal{E},E_d)$ is a \emph{multi-pointed event model}. 
		When $Ag = \{a\}$ for some $a$, we usually refer to the single-agent event model $(E,Q,pre)$ as $(E,Q_a,pre)$.
	\end{definition}
	We will often denote event models by $\mathcal{E}$ independently of whether we refer to an event model $(E,Q,pre)$ or a multi-pointed event model $((E,Q,pre),E_d)$. Their distinction will be clear from context.
	
	\begin{definition}[Product Update]\label{def: product update no post}  Let $\mathcal{M} = (W,R,V)$ be a Kripke model and $\mathcal{E} = (E,Q,pre)$ be an event model. 
		The \emph{product update} of $\mathcal{M}$ with $\mathcal{E}$ is the Kripke model $\mathcal{M} \otimes \mathcal{E} = (W',R',V')$ where:
		
		$W'=\{(w,e)\in W\times E\colon (\mathcal{M},w) \vDash pre(e)\}$,\footnote{We haven't yet defined satisfaction of formulas in $\mathcal{L}$. It's defined in Definition~\ref{def: truth} below, where we again note the standard mutual recursion used in defining DEL~\cite{ditmarsch2007dynamic}.}

		$R'_a=\{((w,e),(v,f))\in W'\times W'\colon (w,v)\in R_a\text{ and } (e,f)\in Q_a\}$,
		
		$V'((w,e))=\{p\in At\cup H\colon w\in V(p)\}$.

		\noindent	Given a pointed Kripke model $(\mathcal{M},w)$ and a multi-pointed event model $(\mathcal{E}, E_d)$, we say that $(\mathcal{E},E_d)$ is \emph{applicable} in $(\mathcal{M},w)$ iff there exists a unique $e \in E_d$ such that $\mathcal{M}, w \vDash pre(e)$. In that case, we define the \emph{product update} of $(\mathcal{M},w)$ with $(\mathcal{E},E_d)$ as the pointed Kripke model $(\mathcal{M},w) \otimes (\mathcal{E},E_d) = (\mathcal{M} \otimes \mathcal{E}, (w,e))$ where $e$ is the unique element of $E_d$ satisfying $(\mathcal{M}, w) \vDash pre(e)$.
	\end{definition}
	\begin{definition}[Satisfaction]\label{def: truth}
		Let $(\mathcal{M},w )= ((W,R,V),w)$
		be a pointed Kripke model. For any $q\in At\cup H, a\in Ag, \varphi \in \mathcal{L}$ and any multi-pointed event model $\mathcal{E}$, satisfaction of $\mathcal{L}$-formulas in $(\mathcal{M},w)$ is given by the following clauses extended with the standard clauses for the propositional connectives:
		\noindent \begin{center}
			\begin{tabular}{lll}
				$(\mathcal{M},w) \vDash q$ & iff & $q\in V(w)$;\tabularnewline
				$(\mathcal{M},w) \vDash B_a\varphi$ & iff & $(\mathcal{M},v) \vDash \varphi$ for all $(w,v)\in R_a$;\tabularnewline
				$(\mathcal{M},w) \vDash [\mathcal{E}]\varphi$ & iff & 
				if $\mathcal{E}$ is applicable in $(\mathcal{M},w)$ then \\
				&&$(\mathcal{M},w) \otimes \mathcal{E} \vDash \varphi$.\tabularnewline
			\end{tabular}
			\par\end{center}	
		We say that a formula $\varphi$ is \emph{valid} if $(\mathcal{M},w) \vDash \varphi$ for all pointed Kripke models $(\mathcal{M},w)$, and in that case we write $\vDash \varphi$.
	\end{definition}
	\begin{example}\label{ex:static}
		Ann and Bob are watching the Invisible Gorilla video~\cite{gorilla_youtube}. Unbeknownst to Ann, Bob has already seen the video, so he knows the correct answer is 15 and that a clearly visible gorilla will pass by. Ann instead has no information about these things, as she has never seen that video. However, she likes riddles and tests of this sort, in which she gets absorbed very easily. Bob knows that, and thus he also knows that she will completely focus on counting the passages only, without realising that there is a gorilla, and thereby thinking to be paying attention to everything happening in the video, just as Bob. This situation is represented in Figure \ref{figure:static}. We have $(\mathcal{M},w) \vDash B_a \mathsf{h}_a g \land \neg \mathsf{h}_a g$: Ann believes she is paying attention to whether there is a gorilla or not, but she isn't.
		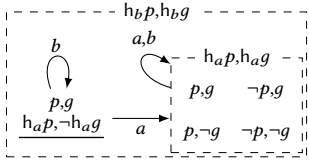
\begin{figure}
			\begin{tikzpicture}\tikzset{deepsquare/.style ={rectangle,draw=black, inner sep=1.5pt, very thin, dashed, minimum height=3pt, minimum width=1pt, text centered}, 
					world/.style={node distance=6pt}, designated/.style={node distance=6pt}
				}
				\node [designated] (!) {$\underline{\s p,g \atop \mathsf{h}_a p, \neg \mathsf{h}_ag}$};
				\path (!) edge [-latex, looseness=7,in=80,out=110] (!) node [above, xshift=-2pt, yshift=22pt] {$\s b$};
				
				\node [world, right=of !, xshift=20pt, yshift=10pt] (1) {$\s p,g$};
				\node [world, below=of 1, yshift=2pt](2) {$\s p, \neg g$};
				\node [world, right=of 1](3) {$\s \neg p,g$};
				\node [world, right=of 2, below=of 3, yshift=2](4) {$\s \neg p,\neg g$};
				\node [deepsquare, fit={($(1) +(0,4mm)$)(2)(3)(4)}](square) {};
				\node[fill=white] (square name 1) at (square.north) {$\s \mathsf{h}_ap, \mathsf{h}_ag$};
				\path (square) edge [-latex,looseness=4.6,in=148,out=165] (square);
				\node [node distance=6pt,above=of square, xshift=-35pt, yshift=-5pt] (loop name 1) {$\s a,b$};
				
				\node[left=-8pt of square, xshift=-2pt, yshift=-5pt] (anchor1) {};
				\path (!) edge [-latex] (anchor1) node [above, xshift=30pt, yshift=-9pt] {$\s a$};
				\node[left=-3pt of square, xshift=2pt, yshift=1.5pt] (anchor1) {};
				
				\node [deepsquare, fit={(square)(!)(square name 1)($(loop name 1) +(0,3mm)$)}](outsquare) {};
				\node [fill=white] (outsquare name) at (outsquare.north) {$\s \mathsf{h}_bp, \mathsf{h}_bg$};
			\end{tikzpicture}
			
			\caption{The pointed Kripke model $(\mathcal{M},w)$. In the figure, $p$ stands for ``the players in the video pass the ball 15 times'', $g$ for ``a clearly visible gorilla crosses the scene''. 
				We use the following conventions. 
				Worlds are represented by sequences of literals true at the world. The model above has 5 worlds, 4 of which are inside the inner dashed box. Designated worlds are underlined. Whenever a world appears inside a dashed box, all the literals in the label of that box are also true in the world---and if the label is underlined, all worlds inside are designated. In this model, $\mathsf{h}_b p$ and $\mathsf{h}_b g$ hold in all worlds, and additionally, $\mathsf{h}_a p$ and $\mathsf{h}_a  g$ hold in the worlds of the inner box. The accessibility relations are represented by labelled arrows. An arrow from (or to) the border of a dashed box means that there is an arrow from (or to) all the events inside the box.  
			}\label{figure:static}
		\end{figure} 
	\end{example}
	\section{Principles for Attention Dynamics}
	In this section, we first present the existing attention model~\cite{bolander2015announcements}. We then propose an alternative representation using our edge principles, introduce a variant, and, finally, generalize to
	multiple propositions (capturing that agents can pay attention to subsets of $At$).
	
	\subsection{The Existing Model and our Version of it}
As in
\cite{bolander2015announcements}, attention is represented as a binary construct where agents can either be paying attention to everything that happens or to nothing.  The language they adopt is as the language above, except for their attention atoms $\mathsf{h}_a$, $a\in Ag$, that are not relativised to propositional formulas. 
The intended meaning of such atoms is that the agent pays attention to everything, so  they can be expressed in our language by
letting $\mathsf{h}_a$, $a\in Ag$, be an abbreviation of the formula $\bigwedge_{p\in At} \mathsf{h}_ap.$
Let $H'=\{\mathsf{h}_a\colon a\in Ag \}$. Then $H'\cup At$ is the set of ``atoms'' on which their language is based. 
The static part of their model is a Kripke model, where it is assumed that agents are \emph{attention introspective}, namely for all $w,v\in W, a\in Ag,$ and $p\in At$, if $(w,v)\in R_a$ then $\mathsf{h}_a (p)\in V(w)$ iff $\mathsf{h}_a(p)\in V(v)$.\longv{\footnote{This assumption is phrased differently in \cite{bolander2015announcements}, as there $\mathsf{h}_a$ is a primitive formula and the valuation function maps atomic formulas into sets of worlds, while here it maps worlds into sets of atomic formulas. Note also that we are generally not assuming our Kripke models to satisfy attention introspection, as the false belief of Ann in Example~\ref{ex:static} makes clear.}}
The dynamics are given by the following event models. 
These event models represent situations in which any formula can be announced, true or false, and attentive agents will come to believe it. 
\begin{definition}[Event Model $\mathcal{E}(\varphi)$, \cite{bolander2015announcements}]\label{def: their action model}
	Given a $\varphi\in \mathcal{L}$, the multi-pointed event model  $\mathcal{E}(\varphi)=((E,Q,pre), E \setminus \{ s_\top\} )$ is defined by: 
	
	$E=\{(i,J)\colon i\in\{0,1\}\text{ and } J\subseteq Ag\}\cup \{s_\top\};$
	
	
	\noindent \hspace{1mm}
	\begin{tabular}{ll}
		$Q_a=$&$\{((i,J),(1,K))\colon i\in \{0,1\}, J,K\subseteq Ag \text{ and }a\in J\}\ \cup $\\
		&$\{((i,J),s_\top )\colon i\in \{0,1\}, J\subseteq Ag \text{ and }a\notin J\};$
	\end{tabular}
	
	$pre\colon E\rightarrow \mathcal{L}$ is defined as follows, for $J\subseteq Ag$:
	\vspace{-1mm}
	\begin{itemize}
		\item [-] $pre((0,J))=\neg \varphi\wedge \bigwedge_{a\in J} \mathsf{h}_a \wedge \bigwedge_{a\not\in J} \neg \mathsf{h}_a;$
		\item [-] $pre((1,J))=\varphi\wedge \bigwedge_{a\in J} \mathsf{h}_a \wedge \bigwedge_{a\not\in J} \neg \mathsf{h}_a;$
		\item [-] $pre(s_\top)=\top$. 
	\end{itemize}
	
\end{definition}
This event model contains $2^{|Ag|+1}+1$ events~\cite{bolander2015announcements}. The preconditions of these events express whether the announced $\varphi$ is true (i.e., whether it occurs positively or negatively in the precondition) and whether each agent $a$ is attentive or not (i.e., whether $\mathsf{h}_a$ occurs positively or negatively in the precondition). We now briefly explain the intuition behind the edges of the model, but refer to~\cite{bolander2015announcements} for more details.
The elements of $Q_a$ of the form $((i,J),(1,K))$ encode the following: Provided that agent 
$a$ is attentive (i.e., $a\in J$), she believes that any event with precondition $\varphi$ could be the actual one.
The elements of $Q_a$ of the form $((i,J),s_\top)$ then encodes:
If instead she is not paying attention (i.e., $a\not\in J$), she keeps the beliefs she had before the announcement (represented by the event $s_\top$ having the precondition $\top$. The $s_\top$ event induces a copy of the original model, thereby modeling the ``skip'' event where nothing happens).

In the following, for any set $S$, we use $id_S$ to denote the identity function on $S$, i.e., $id_S(s) = s$, for all $s \in S$. From now on, most of our event models will be of a particular form where the set of events is a set of (conjunctive) formulas and where preconditions are given by the identify function on $E$, i.e., $pre = id_E$ (meaning that the events are their own preconditions). Our principle-based version of $\mathcal{E}(\varphi)$ is then the following.
\begin{definition}[Principle-Based Event Model $\mathcal{E}'(\varphi)$]\label{a-star-varphi} Given a $\varphi\in \mathcal{L}$, the multi-pointed event model $\mathcal{E}'(\varphi)=((E,Q, id_{E}), E\setminus \{ \top \})$ is: 
	
	$E=\{\psi \wedge  \bigwedge_{a\in J} \mathsf{h}_a \wedge \bigwedge_{a\not\in J} \neg \mathsf{h}_a\colon \psi\in  \{\varphi,\neg\varphi\}, J\subseteq Ag\}\cup\{\top\};$
	%
	
	$Q_a$ is such that $(e,f)\in Q_a$ iff all the following are true:
	\vspace{-1mm}
	\begin{itemize}
		\item[-] \textsc{Basic Attentiveness}: if $\mathsf{h}_a\in e,$ then $\varphi\in f$;
		
		\item[-] \textsc{Inertia}: if $\mathsf{h}_a \not\in e$, then $f = \top$.
		%
	\end{itemize} 
\end{definition}

The \emph{edge principles} of the model above are \textsc{Basic Attentiveness} and \textsc{Inertia}
, describing the conditions under which there is an edge from $e$ to $f$ for agent $a$, that is, what an agent considers possible after the announcement.
By \textsc{Basic Attentiveness}, paying attention implies that, in all events considered possible, the announcement is true---and hence attentive agents believe what is announced. By \textsc{Inertia}, inattentive agents believe nothing happened, namely they maintain the beliefs they had before the announcement was made. 

Note that we have exactly the same set of event preconditions in $\mathcal{E}'(\varphi)$ as in $\mathcal{E}(\varphi)$. The difference is just that we define the events to be their own preconditions, which is possible since all pairs of events have distinct and mutually inconsistent preconditions. It's easy to check that $\mathcal{E}(\varphi)$ and $\mathcal{E}'(\varphi)$ also have the same edges, hence the models are isomorphic. The following proposition shows this.

\begin{proposition}
	$\mathcal{E}(\varphi)$ of Definition 3.1 and  $\mathcal{E}'(\varphi)$ of Definition 3.2 are isomorphic.
\end{proposition}
\begin{proof} 
	We already concluded that the two models have the same set of preconditions, and that all events have distinct preconditions. We then just need to show that for all $a \in Ag$ and all events $e,f \in E$ of $\mathcal{E}(\varphi)$, we have $(e,f) \in Q_a$ in $\mathcal{E}(\varphi)$ iff $(pre(e),pre(f)) \in Q_a$ in $\mathcal{E}'(\varphi)$. To see this, consider first an edge in $Q_a$ of $\mathcal{E}(\varphi)$. It's either of the form $((i,J),(1,K))$ for some $i\in \{0,1\}, J,K\subseteq Ag$ and $a\in J$ or it's of the form $((i,J),s_\top)$ for some $i\in \{0,1\}, J\subseteq Ag$ and $a \not\in J$. According to Definition~3.1, an edge of the first form is an edge from an event with precondition $\neg \varphi\wedge \bigwedge_{a\in J} \mathsf{h}_a \wedge \bigwedge_{a\not\in J} \neg \mathsf{h}_a$ or $\varphi\wedge \bigwedge_{a\in J} \mathsf{h}_a \wedge \bigwedge_{a\not\in J} \neg \mathsf{h}_a$  to an event with precondition $\varphi\wedge \bigwedge_{a\in K} \mathsf{h}_a \wedge \bigwedge_{a\not\in K} \neg \mathsf{h}_a$. Such an edge clearly satisfies \textsc{Basic Attentiveness} (since $\varphi$ is a conjunct of the target of the edge) and \textsc{Inertia} (the condition $a \in J$ for the source event implies that $\mathsf{h}_a$ is contained in the precondition of the source, and hence \textsc{Inertia} holds trivially). This shows that edges in $\mathcal{E}(\varphi)$ of the first type are also edges in $\mathcal{E}'(\varphi)$. The argument for edges of the second type is similar, but here the condition of the source is $a \not\in J$, meaning that \textsc{Basic Attentiveness} instead is trivial, and we only need to show \textsc{Inertia}. According to Definition~3.1, an edge of the second type is an edge from an event with precondition $\neg \varphi\wedge \bigwedge_{a\in J} \mathsf{h}_a \wedge \bigwedge_{a\not\in J} \neg \mathsf{h}_a$ or $\varphi\wedge \bigwedge_{a\in J} \mathsf{h}_a \wedge \bigwedge_{a\not\in J} \neg \mathsf{h}_a$ (as before) to an event with precondition $\top$. Since $a \not\in J$, we have that $\mathsf{h}_a$ is not contained in the precondition of the source event. \textsc{Inertia} then requires that the precondition of the target is $\top$, but that we already concluded. So \textsc{Inertia} holds, as required.
	
	For the other direction, we start with an edge $(e,f) \in Q_a$ of $\mathcal{E}'(\varphi)$ satisfying both \textsc{Basic Attentiveness} and \textsc{Inertia}, and show that it is of one of the two types in $\mathcal{E}(\varphi)$. We split into cases depending on whether $\mathsf{h}_a \in e$ or not. If $\mathsf{h}_a \in e$, then by \textsc{Basic Attentiveness}, $\varphi \in f$. Let $J$ denote the set of agents for which $\mathsf{h}_a$ occurs positively in $e$, and let $K$ denote the same set for $f$. Since  $\mathsf{h}_a \in e$, we get $a \in J$. Let $i=0$ if $\neg \varphi$ occurs in $e$, otherwise let $i=1$. Then $e = pre((i,J))$, using the notation from Definition~3.1. Since $\varphi \in f$, we have that $f = pre((1,K))$. By Definition~3.1, $Q_a$ contains an edge from $(i,J)$ to $(1,K)$. This covers the case where $\mathsf{h}_a \in e$. Consider now the case $\mathsf{h}_a \not\in e$. By \textsc{Inertia}, $f = \top$. Define $J$ and $i$ as before from $e$. Then, as before, $e = pre((i,J))$. Since $f = \top$, $f = pre(s_\top)$. By Definition~3.1, $Q_a$ contains an edge from $(i,J)$ to $s_\top$, and we're done.  
\end{proof}


Compare the edge specification from $\mathcal{E}(\varphi)$ with the one from $\mathcal{E}'(\varphi)$. We are defining the same set of edges, but whereas the definition of $Q_a$ in $\mathcal{E}(\varphi)$ does not make it immediately clear what those edges are encoding, we believe that our definition of $Q_a$ in $\mathcal{E}'(\varphi)$ does. It is simply two basic principles, one specifying what events are considered possible by the agents paying attention (\textsc{Basic Attentiveness}), and another specifying the same for those not paying attention 
(\textsc{Inertia}). 
Even though from a technical viewpoint it is not a big step to introduce such principles, we find it helpful to be able to specify the relevant event models in a clear and concise manner. 
This makes it easier to use the model and build on it---as should become evident when we later generalise the event model. 
\subsubsection{Modified model} 
We now introduce a variant of the event model
$\mathcal{E}'(\varphi)$ from 
Def. \ref{a-star-varphi}, 
one that is more appropriate for the types of scenarios that we would like to be able to model. 
\paragraph{Truthful announcements} As the present work aims at modeling (noise-free) attention to external stimuli from the environment, in particular visual attention, the first assumption we give up is that announcements may be false. 
More precisely, we assume that 
if an agent pays attention to $p$ and the truth-value of $p$ is being revealed, then the agent sees the true truth-value of $p$.
The new event model for announcing $\varphi$ should then only contain events where $\varphi$ is true:
	%

\begin{center}
	$E=\{\varphi \wedge  \bigwedge_{a\in J} \mathsf{h}_a \wedge \bigwedge_{\mathsf{h}_a\not\in J} \neg \mathsf{h}_a\colon  J\subseteq Ag\}\cup\{\top\}.$
\end{center}
\paragraph{Learning that you were attentive}
An assumption we have already given up is attention introspection, so in our models the agents may falsely believe to be paying attention (see Example \ref{ex:static}). 
	In this setting, 
	it is very plausible to assume that, besides learning what the true event is, attentive agents also learn that they were attentive.
This does not happen in the event model $\mathcal{E}'(\varphi)$. We thus substitute \textsc{Basic Attentiveness} with the following principle:
\begin{itemize}
	\item [-] \textsc{Attentiveness}: if $\mathsf{h}_a\in e,$ then $\mathsf{h}_a,\varphi \in f$.
\end{itemize}
Summing up, the event model where announcements are truthful and attentive agents learn that they paid attention, looks as follows.
\begin{definition}[Truthful and Introspective Event Model $\mathcal{E}''(\varphi)$] \label{truthful and introspective}
	Given $\varphi\in \mathcal{L}$, the multi-pointed event model $\mathcal{E}''(\varphi)=((E, Q, id_{E}), E\setminus \{ \top \})$ is defined by:
	
	$E=\{\varphi \wedge  \bigwedge_{a\in J} \mathsf{h}_a \wedge \bigwedge_{a\not \in J} \neg \mathsf{h}_a\colon  J\subseteq Ag\}\cup\{\top\}$;
	
	\smallskip
	$Q_a$ is such that $(e,f)\in Q_a$ iff all the following are true:
	\begin{itemize}
		\item[-] \textsc{Attentiveness}: if $\mathsf{h}_a\in e,$ then $\mathsf{h}_a,\varphi \in f$;
		\item[-] \textsc{Inertia}: if $\mathsf{h}_a \not\in e$, then $f = \top$;
	\end{itemize} 
	
\end{definition}	
The event model 
$\mathcal{E}''(p \land g)$ with $Ag = \{a,b\}$  is shown in Figure~\ref{fig: event introspective}. 
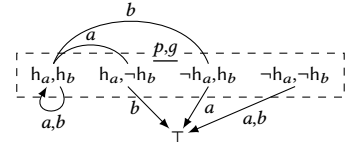
\begin{figure}
	\begin{tikzpicture}\tikzset{deepsquare/.style ={rectangle,draw=black, inner sep=1pt, thin, dashed, minimum height=3pt, minimum width=1pt, text centered}, 
			world/.style={node distance=6pt,inner sep=1pt},auto}
		
		\node [world] (1) {${\s \mathsf{h}_a,  \mathsf{h}_b}$};
		\node [world, right=of 1] (2) {${\s \mathsf{h}_a, \neg \mathsf{h}_b}$};
		\node [world, right=of 2] (3) {${\s \neg \mathsf{h}_a, \mathsf{h}_b}$};
		\node [world, right=of 3] (4) {${\s \neg \mathsf{h}_a,  \neg \mathsf{h}_b}$};
		\node [world, below=of 1, xshift=47pt, yshift=-10pt] (5) {$\s \top$};
		
		\path (1) edge [-latex,loop,looseness=7,in=240,out=300] node[yshift=0.5mm] {$\s a,b$}  (1);
		
		\path (2.south) edge [-latex]  node[left] {$\s b$} (5);
		
		\path (3.south) edge [-latex]  node[yshift=1.5mm]  {$\s a$} (5);
		
		
		\path (4.south) edge [-latex] node[xshift=-1mm,yshift=1mm]  {$\s a,b$} (5);
		
		\path[-] (1.north) edge[bend left=60] node[yshift=-0.5mm] {$\s a$} (2.north);		
		
		\path[-] (1.north) edge[bend left=70] node[yshift=-0.5mm] {$\s b$} (3.north) ;
		
		\node[deepsquare, inner sep=4pt, fit={(1)(2)(3)(4)}] (announcement) {};
		\node[fill=white,inner sep=1pt,xshift=-2mm] (square name 1) at (announcement.north) {\underline{$\s p, g$}};
	\end{tikzpicture}
	\caption{The event model $\mathcal{E}''(p\wedge g)$ with $Ag = \{a,b\}$. As our event models will have conjunctive preconditions, all distinct, and our events are their own preconditions, we can represent events by lists of formulas, the formulas contained in the event precondition. 
		All other conventions are as for Kripke models (see Fig.~\ref{figure:static}).	
	}
	\label{fig: event introspective}
\end{figure}
\subsection{Event Models for Propositional Attention}

In this section, we introduce event models for agents that only pay attention to  subsets of $At$.  
As our main aim  is to model attention to external stimuli, we are interested in modeling the ``announcement'' of a conjunction of literals $(\neg) p_1 \wedge \dots \wedge (\neg) p_n$, which we interpret as the parallel exposure to multiple stimuli (the truth value of all $p_i$ being revealed concurrently). It could for instance be that we see a video that has 15 ball passes and a gorilla passing by, and that would correspond to the ``announcement'' $p \land g$, cf.\ Example	~\ref{ex:static}. 
\begin{definition}[Propositional Attention Event Model $\mathcal{F}(\varphi)$]\label{e-varphi} Let $\varphi=\ell(p_1)\wedge \dots \wedge \ell(p_n) \in \mathcal{L}$, where for each $p_i$, either $\ell(p_i)=p_i$ or $\ell(p_i)=\neg p_i$. 
	The multi-pointed event model $\mathcal{F}(\varphi)=((E,Q,id_{E}), E_d)$ is defined by:
	%
	\begin{multline*}
			E=\{\bigwedge_{p \in S} \ell(p) \wedge  \bigwedge_{a \in Ag} \bigl(\bigwedge_{p \in X_a} \mathsf{h}_a p \wedge \bigwedge_{p \in S \setminus X_a} \neg \mathsf{h}_a p \bigr)  \colon \\S \subseteq \mathit{At}(\varphi) \text{ and for all } a \in Ag,X_a \subseteq S \}
		\end{multline*}
	\smallskip
	$Q_a$ is such that $(e,f)\in Q_a$ iff all the following hold for all $p$: 
	
	\vspace{-1mm}
	\begin{itemize}
		\item[-] \textsc{Attentiveness}: if $\mathsf{h}_a p\in e$ then $\mathsf{h}_a p, \ell(p)\in f$;
		\item[-] \textsc{Inertia}: if $\mathsf{h}_a p\notin e$ then $\ell(p)\not\in f;$
	\end{itemize}  
	
	%
	%
			$E_d=\{\psi\in E\colon \ell(p)\in \psi, \text{ for all } \ell(p)\in \varphi\}$. 
		\end{definition}
		
		In  $\mathcal{F}(\varphi)$ we have, for each subset of literals in $\varphi$, an event containing those literals in the precondition. For those literals, the event also specifies whether each agent is paying attention to it or not. In this way, events account for all possible configurations of attention to any subset of the announcement and for the learning of truthful information regarding it. The edges are again given by two simple principles. 
		\textsc{Attentiveness} states that if an agent pays attention to a specific atom, then she learns the literal in the announcement corresponding to it and that she was paying attention to it. \textsc{Inertia} says that if an agent doesn't pay attention to an atom, then she will not learn anything about it. As we take announcements as truthful revelations, the set of designated events only contains events where all the announced literals are true.
		The event model $\mathcal{F}(\varphi)$ with $\varphi = p \land g$ and $Ag = \{a,b\}$ is shown in Figure~\ref{figure:monster}.
		
		
		%

		\begin{figure}
		\begin{center}
		\includegraphics[width=0.5\textwidth]{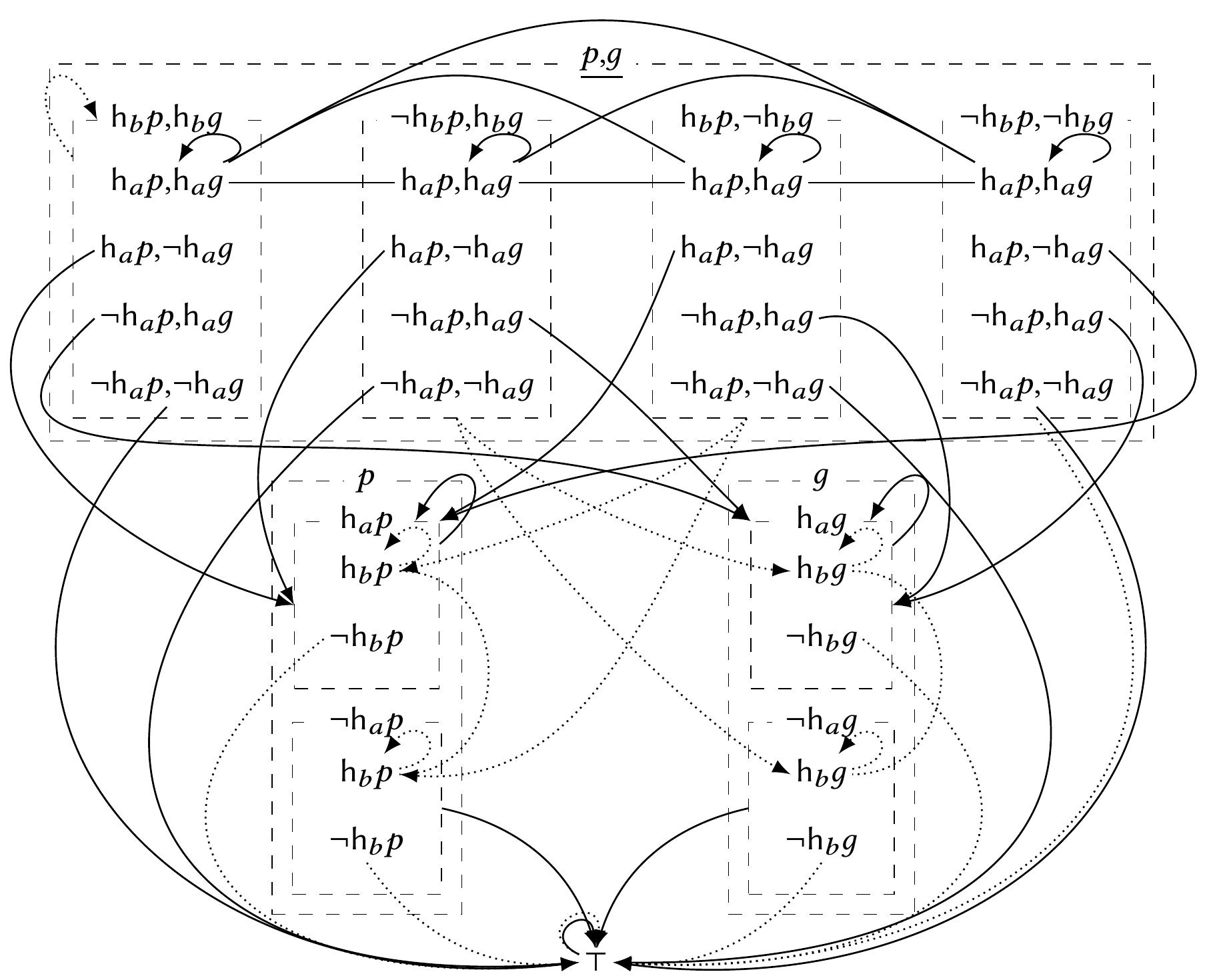}
			\caption{The event model $\mathcal{F}(p\wedge g)$ with $Ag = \{a,b\}$. Solid arrows are for agent $a$, dotted for agent $b$. A small Python program for computing the above edges from the edge principles can be found here: \url{https://tinyurl.com/5ekjmsud}.  
			}	\label{figure:monster}
			\end{center}
		\end{figure}
		
		\begin{example} 
			Continuing Example \ref{ex:static}, Ann and Bob have finished watching the Invisible Gorilla video (event model $\mathcal{F}(p\wedge g)$). As Bob expected, Ann learns that there are 15 ball passes, 
			but she still doesn't know anything about whether there is a gorilla in the video, and believes Bob is in the same situation as herself. 
			The pointed Kripke model $(\mathcal{M}',w') = 
			(\mathcal{M},w)\otimes\mathcal{F}(p\wedge g)$ in Figure \ref{figure:update} (left) represents the situation after exposure to the video, i.e., after the revelation of $p\wedge g$.  Ann has only learnt about $p$ and 
			still has no information about $g$. We thus have $(\mathcal{M}',w') \vDash B_a p\wedge \neg B_a g \land \neg B_a \neg g$. Moreover, she wrongly believes Bob too hasn't received any information about the gorilla, so $(\mathcal{M}',w') \vDash B_b g \land B_a(\neg B_b g \land \neg B_b \neg g)$. 
			
			\begin{figure}
				\begin{tikzpicture}\tikzset{deepsquare/.style ={rectangle,draw=black, inner sep=1.5pt, very thin, dashed, minimum height=3pt, minimum width=1pt, text centered}, 
						world/.style={node distance=6pt}, designated/.style={node distance=6pt}
					}
					\node [designated] (!) {$\underline{\s p,g \atop \mathsf{h}_a p, \neg \mathsf{h}_ag}$};
					\path (!) edge [-latex, looseness=7,in=80,out=110] (!) node [above, xshift=-2pt, yshift=22pt] {$\s b$};
					
					\node [world, right=of !, xshift=20pt, yshift=10pt] (1) {$\s p,g$};
					\node [world, below=of 1, yshift=2pt](2) {$\s p, \neg g$};
					\node [deepsquare, fit={($(1) +(0,4mm)$)(2)}](square) {};
					\node[fill=white] (square name 1) at (square.north) {$\s \mathsf{h}_ap, \mathsf{h}_ag$};
					\path (square) edge [-latex,looseness=5,in=140,out=165] (square);
					\node [node distance=6pt, xshift=30pt, yshift=22pt] (loop name 1) {$\s a,b$};
					
					\node[left=-8pt of square, xshift=-2pt, yshift=-5pt] (anchor1) {};
					\path (!) edge [-latex] (anchor1) node [above, xshift=30pt, yshift=-9pt] {$\s a$};
					\node[left=-3pt of square, xshift=2pt, yshift=1.5pt] (anchor1) {};
					
					\node [deepsquare, fit={(square)(!)(square name 1)($(loop name 1) +(0,3mm)$)}](outsquare) {};
					\node [fill=white] (outsquare name) at (outsquare.north) {$\s \mathsf{h}_bp, \mathsf{h}_bg$};
				\end{tikzpicture} \qquad \quad
				\begin{tikzpicture}\tikzset{deepsquare/.style ={rectangle,draw=black, inner sep=1.5pt, very thin, dashed, minimum height=3pt, minimum width=1pt, text centered}, 
						world/.style={node distance=6pt}, designated/.style={node distance=6pt}
					}
					\node [designated] (!) {$\underline{\s p,g \atop \mathsf{h}_a p, \neg \mathsf{h}_ag}$};
					\path (!) edge [-latex, looseness=7,in=80,out=110] (!) node [above, xshift=-2pt, yshift=22pt] {$\s b$};

					\node [world, right=of !, xshift=15pt, yshift=0.5pt] (1) {$\s p, \neg g, \atop  \mathsf{h}_a p,\mathsf{h}_a g$};
					\path (1) edge [-latex, looseness=7,in=80,out=110]  node [above] {$\s a,b$} (1);

					
					\node[left=-8pt of square, xshift=-2pt, yshift=-5pt] (anchor1) {};
					\path (!) edge [-latex] (anchor1) node [above, xshift=30pt, yshift=-9pt] {$\s a$};
					\node[left=-3pt of square, xshift=2pt, yshift=1.5pt] (anchor1) {};
					
					\node [deepsquare, fit={(square)(!)(square name 1)($(loop name 1) +(1.5,3mm)$)}](outsquare) {};
					\node [fill=white] (outsquare name) at (outsquare.north) {$\s \mathsf{h}_bp, \mathsf{h}_bg$};
				\end{tikzpicture}

				\caption{Pointed Kripke models $(\mathcal{M}',w') = (\mathcal{M},w)\otimes\mathcal{F}(p\wedge g)$ (left) and 
					$(\mathcal{M}'',w'') = (\mathcal{M},w)\otimes\mathcal{E}(p\wedge g, d)$ (right), where the default map $d$ is $d_a(p)=d_b(p)=\top$ and $d_a(g)=d_b(g)=\neg g$. Worlds inaccessible from the designated world are not shown.} \label{figure:update}
			\end{figure}
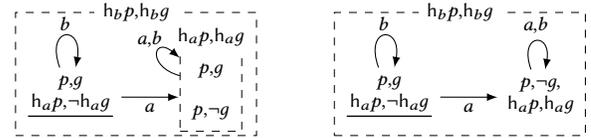 
		\end{example}
		
		%
		%
		\section{Axiomatization} 
	We move to the axiomatization of our logic and show that it is sound and complete. The axiomatization is given by the set of axioms and inference rule of Table \ref{tab:logic}. It comprises standard axioms and inference rules for normal modal logic as well as reduction axioms. All propositional reduction axioms in Table \ref{tab:logic} are for  state-eliminating updates, in that they are relativized to the announced $\varphi$. Then, the only non-standard axiom we introduce is the one expressing the consequences of attention-dependent announcements for what concerns agents' beliefs. Where $\varphi=\ell(p_1)\wedge\dots\wedge \ell(p_n)$ is the announced formula, the axiom is the following:
	\begin{multline*}
		[\mathcal{F}(\varphi)]B_a\psi \leftrightarrow (\varphi \rightarrow \bigvee_{S\subseteq At(\varphi)}\bigl( \bigwedge_{p\in S}\mathsf{h}_ap \wedge \bigwedge_{p\in At(\varphi)\setminus S}\neg \mathsf{h}_ap)\\\rightarrow B_a([\mathcal{F}(\bigwedge_{p\in S} \ell(p)) ]\psi))\bigr)
	\end{multline*}
	The axiom can be read as saying that after exposure to the revelation of $\varphi$,
	agent $a$ believes that only the conjunction of literals from $\varphi$ to which she was paying attention to has been revealed. 
	\begin{table} 
		\caption{\label{tab:logic}The logic of propositional attention $\Lambda$. It is assumed that $\varphi = \ell(p_1)\wedge \dots \wedge \ell(p_n)$ for some literals $\ell(p_i)$, $i = 1,\dots,n$.}
		\begin{tabular}{l}\toprule
			All propositional tautologies\\
			$B_{a}(\varphi\rightarrow \psi)\rightarrow (B_{a}\varphi\rightarrow B_a\psi)$ \\
			$[\mathcal{F}(\varphi)] p \leftrightarrow (\varphi\rightarrow p)$ \\
			$[\mathcal{F}(\varphi)] \neg \psi \leftrightarrow( \varphi \rightarrow \neg [\mathcal{F}(\varphi)]\psi)$ \\
			$[\mathcal{F}(\varphi)](\psi\wedge\chi)\leftrightarrow([\mathcal{F}(\varphi)]\psi\wedge[\mathcal{F}(\varphi)]\chi)$\\
			$[\mathcal{F}(\varphi)]B_a\psi \leftrightarrow (\varphi \rightarrow \bigvee_{S\subseteq At(\varphi)}(( \bigwedge_{p\in S} \mathsf{h}_ap \wedge \bigwedge_{p\in At(\varphi)\setminus S}\neg \mathsf{h}_ap)$ \\ \hfill$\rightarrow B_a([\mathcal{F}(\bigwedge_{p\in S} \ell(p)) ]\psi)))$\\
			From $\varphi$ and $\varphi\rightarrow\psi$, infer $\psi$ \\
			From $\varphi$ infer $B_a\varphi$\\
			From $\varphi\leftrightarrow \psi$, infer $\chi[\varphi\slash p]\leftrightarrow \chi[\psi\slash p]$\footnote{This is standard notation for substitution, although it looks similar to the notation for the dynamic modality.}\\
			\bottomrule
		\end{tabular}
	\end{table}

To prove soundness and completeness of the axiomatization in Table~\ref{tab:logic}, we will use the following lemma, which shows that updating a Kripke model $(\mathcal{M},w)$ with event model $\mathcal{F}(\varphi)$ where $\varphi = \ell(p_1)\wedge \dots \wedge \ell(p_n)$ is announced, or updating it with $\mathcal{F}(\bigwedge_{p\in S}\ell(p))$, where $\bigwedge_{p\in S}\ell(p)$ are the literals from $\varphi$ that agent $a$ is paying attention to at $(\mathcal{M},w)$, yields updates $(\mathcal{M},w)\otimes\mathcal{F}(\varphi)$ and $(\mathcal{M},w)\otimes\mathcal{F}(\bigwedge_{p\in S}\ell(p))$ that are bisimilar from agent $a$'s perspective.

In what follows, events containing all the announced literals will be called ``maximal''. We will use notation $Q_a[e]$ to indicate the states (worlds or events) that are $Q_a$-accessible from $e$, i.e., $Q_a[e]=\{f \colon (e,f)\in Q_a\}$. Lastly, if $\varphi,\psi\in \mathcal{L}$ are conjunctions of literals, we will say that $\psi\in \varphi$ iff  $Lit(\psi)\subseteq Lit(\varphi)$. In that case, we will also say that $\varphi$ contains $\psi$.
		\begin{lemma}\label{cor1} For any pointed Kripke model $(\mathcal{M},w)$ with $(\mathcal{M},w)\vDash\varphi$, and any $a\in Ag$, consider the unique $S\subseteq At(\varphi)$ that is such that $(\mathcal{M},w)\vDash (\bigwedge_{p\in S} \mathsf{h}_ap \wedge \bigwedge_{p\in At(\varphi)\setminus S}\neg \mathsf{h}_ap)$. Then, the updated models $(\mathcal{M},w)\otimes \mathcal{F}(\bigwedge_{p\in S}\ell(p))=((W^{\varphi_S}, R^{\varphi_S}, V^{\varphi_S}),(w,e'))$ and $(\mathcal{M},w)\otimes \mathcal{F}(\varphi)=((W^{\varphi}, R^{\varphi}, V^{\varphi}),(w,e))$ are such that:
		\begin{enumerate}
			\item [(1)] $R_a^{\varphi}[(w,e)]= R_a^{\varphi_S}[(w,e')]$;
		\item [(2)] For all $(v,f)\in R_a^{\varphi}[(w,e)]$, there exists a bisimulation between $(\mathcal{M}^\varphi,(v,f))$ and $(\mathcal{M}^{\varphi_S},(v,f))$, notation $(\mathcal{M}^\varphi,(v,f))\leftrightarroweq(\mathcal{M}^{\varphi_S},(v,f))$;\footnote{The notion of bisimulation for Kripke model is standard, see e.g., \cite{blac.ea:moda}. }
		\end{enumerate}
		\end{lemma}
		\begin{proof} Let $(\mathcal{M},w)=((W,R,V),w)$ be a pointed Kripke model. We will use the same notation as in the previous proof for $\varphi_S$, for $\mathcal{F}(\varphi)=((E,Q,pre),E_d)$ and $\mathcal{F}(\varphi_S)=((E',Q',pre'),E'_d)$. For the $\varphi$- and $\varphi_{S}$-updates of $(\mathcal{M},w)$ we will use the notation introduced in the statement of the lemma, if not otherwise stated.
			
			Assume that $(\mathcal{M},w)\vDash\varphi$. Then $\mathcal{F}(\varphi)$ and $\mathcal{F}(\varphi_S)$ are applicable to $(\mathcal{M},w)$, so $(\mathcal{M},w)\otimes \mathcal{F}(\varphi)=(\mathcal{M}^\varphi,(w,e))$ and $(\mathcal{M},w)\otimes \mathcal{F}(\varphi_S)=(\mathcal{M}^{\varphi_S},(w,e'))$ exist. 
			Now let $S\subseteq At(\varphi)$ be the unique $S$ that is such that $(\mathcal{M},w)\vDash (\bigwedge_{p\in S} \mathsf{h}_ap \wedge \bigwedge_{p\in At(\varphi)\setminus S}\neg \mathsf{h}_ap)$, for some $a\in Ag$.
			
			(1) We first show that $R_a^{\varphi}[(w,e)]= R_a^{\varphi_S}[(w,e')]$, proving the two inclusions separately.
			
			$(\Rightarrow)$ Let $(v,f)\in R^{\varphi}[(w,e)]$. This means that $v\in R_a[w]$ and $f\in Q_a[e]$. Then, to reach the desired result that $(v,f)\in R^{\varphi_S}[(w,e')]$, we only need to show that $f\in Q'_a[e']$, as then we would have that $v\in R_a[w]$ and $f\in Q'_a[e']$, and since $(\mathcal{M},v)\vDash pre(f)$ then $(v,f)\in W^{\varphi_S}$ and we could conclude that $(v,f)\in R^{\varphi_S}[(w,e')]$. We show that $f\in Q'_a[e']$ by showing that $f$ is such that $f\in E'$ and that it satisfies the requirements that \textsc{Attentiveness} and \textsc{Inertia} pose to belong to $Q'_a[e']$, i.e., it contains the needed formulas.
			
			So let's first see what formulas $f$ contains.   By initial assumption, $(\mathcal{M},w)\vDash (\bigwedge_{p\in S} \mathsf{h}_ap \wedge \bigwedge_{p\in At(\varphi)\setminus S}\neg \mathsf{h}_ap)$. As $(w,e)\in W^\varphi$, then by product update definition and maximality of $e$, it holds that $(\bigwedge_{p\in S} \mathsf{h}_ap \wedge \bigwedge_{p\in At(\varphi)\setminus S}\neg \mathsf{h}_ap)\in e$. Then by \textsc{Attentiveness} and $\bigwedge_{p\in S} \mathsf{h}_ap\in e$, we know that $\bigwedge_{p\in S} (\ell(p)\wedge \mathsf{h}_ap)\in f$. Moreover, as $\bigwedge_{p\in At(\varphi)\setminus S} \neg \mathsf{h}_ap\in e$, then by def. of event model for propositional attention (in particular by definition of its set of events) for all $p\in At(\varphi)\setminus S$, $\mathsf{h}_ap\notin e$ and so by \textsc{Inertia}, $f$ doesn't contain $\ell(p)$, for all $p\in At(\varphi)\setminus S$, which then also means that $\mathsf{h}_ap\notin f$ for all such $p\in At(\varphi)\setminus S$, by def. of event models for propositional attention. Hence, $f$ is such that $\bigwedge_{p\in S} (\ell(p)\wedge \mathsf{h}_ap)\in f$ as well as, for all $p\in At(\varphi)\setminus S, \mathsf{h}_ap,\ell(p)\notin f$. 
			
			Now let's see what is required to belong to $Q'_a[e']$. Since by initial assumption $(\mathcal{M},w)\vDash \bigwedge_{p\in S} \mathsf{h}_ap$ for some $S\subseteq At(\varphi)$ and since $(w,e')\in W^{\varphi_S}$, then by product update definition and maximality of $e'$, it holds that $\bigwedge_{p\in S} \mathsf{h}_ap \in e'$. Then we can use \textsc{Attentiveness} to see that in order to belong to $Q'_a[e']$ an event must contain $\bigwedge_{p\in S} (\ell(p)\wedge \mathsf{h}_ap)$. Moreover, since all events in $Q'_a[e']$ are events from $\mathcal{F}(\varphi_S)$ then they contain only literals and attention atoms from $\varphi_S$. So to belong to $Q'_a[e']$, and thus to $E'$, an event must not contain $\ell(p),\mathsf{h}_ap$, for all $p\in At(\varphi)\setminus S$. Hence, to belong to $Q'_a[e']$, an event $f'$ must be such $\bigwedge_{p\in S} (\ell(p)\wedge \mathsf{h}_ap)\in f'$ as well as, for all $p\in At(\varphi)\setminus S, \mathsf{h}_ap\notin f'$ and $\ell(p)\notin f'$. This is exactly what we have with $f$ and since \textsc{Attentiveness} and \textsc{Inertia} are the only requirements to satisfy to be part of $Q'_a[e']$, then and $f\in Q'_a[e']$. 
			
			Hence, we have that if $f\in Q_a[e]$ then $f\in Q'_a[e']$. Above we assumed that $(v,f)\in R^{\varphi}[(w,e)]$, i.e., that $v\in R_a[w]$ and $f\in Q_a[e]$. This now implies that $v\in R_a[w]$ and $f\in Q'_a[e']$, and since $(\mathcal{M},v)\vDash pre(f)$ and so $(v,f)\in W^{\varphi_S}$, then by def. of product update that $(v,f)\in R_a^{\varphi_S}[(w,e')]$.
			
			$(\Leftarrow)$ This proof proceed analogously to the above proof of the other inclusion.
			
We can conclude that $R_a^{\varphi}[(w,e)]= R_a^{\varphi_S}[(w,e')]$. \\
			(2) We now show that for all $(v,f)\in R_a^{\varphi}[(w,e)]$,  $(\mathcal{M}^\varphi,(v,f))\leftrightarroweq(\mathcal{M}^{\varphi_S},(v,f))$. Consider a bisimulation $\mathcal{Z}\subseteq (W^{\varphi}\times W^{\varphi_S})$ defined by $(u',g')\in \mathcal{Z}[(u,g)]$ iff $u=u'$ and $g=g'$ (recall that events are formulas, so $g=g'$ means that their preconditions are the same). We show that it satisfies the three requirements of bisimulations for Kripke models. Let $(u',g')\in \mathcal{Z}[(u,g)]$. 
			
			[Atom]:  Since $u=u'$, then clearly $(u,g),(u',g')$ satisfy the same atomic formulas, by def. of product update.
			
			[Forth]: Let $(t,h)\in R_b^\varphi[(u,g)]$, for some $b\in Ag$. We want to show that there exists a state $(t',h')\in W^{\varphi_S}$ such that $(t',h')\in R_b^{\varphi_S}[(u',g')]$ and $(t',h')\in \mathcal{Z}[(t,h)]$. 
			By def. of product update, since $(t,h)\in R_b^\varphi[(u,g)]$, then $t\in R_b[u]$ and $h\in Q_b[g]$. As by initial assumption $(u',g')\in \mathcal{Z}[(u,g)]$, then $g=g'$, that is, $g$ and $g'$ are the same formula. This implies, by \textsc{Attentiveness} and \textsc{Inertia} and by $h\in Q_b[g]$, that $h\in Q'_b[g']$ (the argument to see that this holds proceeds analogously to the argument given in (1), to show that if $f\in Q_a[e]$ then $f\in Q'_a[e']$). Moreover, as $(u',g')\in \mathcal{Z}[(u,g)]$ then $u=u'$, and since $t\in R_b[u]$ then clearly $t\in R_b[u']$. Since $(t,h) \in R^{\varphi}_b[(u,g)]$ then $(t,h) \in W^\varphi$, implying that the precondition of $h$ is satisfied in $t$. We now have $h\in Q'_b[g']$, $t\in R_b[u']$ and that the precondition of $h$ is satisfied in $t$ which, by product update definition, implies $(t,h) \in W^{\varphi_S}$ and $(t,h)\in R_b^{\varphi_S}[(u',g')]$. Letting $t'=t$ and $h'=h$, this proves the required.    
			
			[Back]: Analogous to the Forth condition.
			
			As by (1) we have $ R^\varphi_a[(w,e)]=R^{\varphi_S}_a[(w,e')]$, then by choice of bisimulation relation $\mathcal{Z}$ we can conclude that for all $(v,f)\in R^\varphi_a[(w,e)], (\mathcal{M}^\varphi,(v,f))\leftrightarroweq(\mathcal{M}^{\varphi_S},(v,f))$.
		\end{proof}
			\begin{theorem} \label{sound and complete} 
			The axiomatization in Tbl.~\ref{tab:logic} is sound and complete. 
		\end{theorem}
		\begin{proof} 
			\emph{Completeness:} It proceeds by usual reduction arguments \cite{ditmarsch2007dynamic}. 
			\emph{Soundness:} We show that axioms and inferences rules from Table~\ref{tab:logic} are valid. Axioms and inference rules for normal modal logic are valid in pointed Kripke models, by standard results \cite{blac.ea:moda}. As our product update is of the state-eliminating kind, the propositional reduction axioms are valid  \cite{ditmarsch2007dynamic}. Thus, we only need to show the validity of the reduction axiom for attention-based belief updates. We prove the two directions separately.
			
			Let $(\mathcal{M},w)=((W,R,V),w)$ be a pointed Kripke model. We use the same notation as in the previous proof for $\varphi_S$, for $\mathcal{F}(\varphi)$ and $\mathcal{F}(\varphi_S)$, and for the updates $(\mathcal{M}^{\varphi}, (w,e))$ and $(\mathcal{M}^{\varphi_S}, (w,e'))$.
			
			($\Rightarrow$) In this direction we want to prove that if we assume $(\mathcal{M},w)\vDash [\mathcal{F}(\varphi)]B_a\psi$ for some arbitrary $a\in Ag$, then it follows that $(\mathcal{M},w)\vDash \varphi\rightarrow \bigvee_{S\subseteq At(\varphi)} ((\bigwedge_{p\in S} \mathsf{h}_ap\wedge \bigwedge_{p\in At(\varphi)\setminus S} \neg \mathsf{h}_ap)\rightarrow B_a (\bigwedge_{p\in S} \ell(p)\rightarrow [\mathcal{F}(\varphi_S)]\psi))$. We will show that the claim follows straightforwardly from Lemma \ref{cor1}.
			Let $(\mathcal{M},w)\vDash [\mathcal{F}(\varphi)]B_a\psi$  for some arbitrary $a\in Ag$, let $(\mathcal{M},w)\vDash \varphi$ and let $S\subseteq At(\varphi)$ be the unique $S$ such that $(\mathcal{M},w)\vDash \bigwedge_{p\in S} \mathsf{h}_ap\wedge \bigwedge_{p\in At(\varphi)\setminus S} \neg \mathsf{h}_ap$. As $(\mathcal{M},w)\vDash \varphi$ then $\mathcal{F}(\varphi)$ is applicable in $(\mathcal{M},w)$ and $(\mathcal{M}^{\varphi}, (w,e))$ and $(\mathcal{M}^{\varphi_S}, (w,e'))$ exist.
			As $(\mathcal{M},w)\vDash [\mathcal{F}(\varphi)]B_a\psi$, then we know, by semantics of the dynamic modality and by applicability of $\mathcal{F}(\varphi)$ to $(\mathcal{M},w)$, that $(\mathcal{M}^\varphi,(w,e))\vDash B_a\psi$, and so, by semantics of belief modality, for all $(v,f)\in R^\varphi_a[(w,e)], (\mathcal{M}^\varphi,(v,f))\vDash \psi$. As our assumptions here are the same assumptions made in Lemma \ref{cor1}, we can then use that lemma to obtain that $R_a^{\varphi}[(w,e)]= R_a^{\varphi_S}[(w,e')]$ and that for all $(v,f)\in R_a^{\varphi}[(w,e)]$, $(\mathcal{M}^\varphi,(v,f))\leftrightarroweq(\mathcal{M}^{\varphi_S},(v,f))$. By standard results, bisimulation implies modal equivalence (see e.g., \cite{blac.ea:moda}). Hence, 
			it follows that for all $(v,f)\in R^{\varphi_S}_a[(w,e')]$, $(\mathcal{M}^{\varphi_S},(v,f))\vDash \psi$. This means that for all $v\in R_a[w]$ and all $f\in Q'_a[e']$ that are such that $(v,f)\in W^{\varphi_S}$, $(\mathcal{M}^{\varphi_{S}},(v,f))\vDash \psi$.
			
			Now we have two cases: for any $v\in R_a[w]$, either $\mathcal{F}(\varphi_S)$ is applicable in $(\mathcal{M},v)$ or it is not. If it is not applicable, we can directly conclude that $(\mathcal{M},v)\vDash [\mathcal{F}(\varphi_S)]\psi$, by semantics of dynamic modality, and since this holds for an arbitrary $v\in R_a[w]$, then $(\mathcal{M},w)\vDash B_a([\mathcal{F}(\varphi_S)]\psi)$, by semantics of belief modality. Now consider the case in which $\mathcal{F}(\varphi_S)$ is applicable in $(\mathcal{M},v)$. In this case, we need to show that for any  $f\in Q'_a[e']$ with $(v,f)\in W^\varphi$, $f$ is maximal, i.e., $f\in E'_d$, to then be able to infer, by semantics of dynamic modality, that for all $v\in R_a[w]$, $(\mathcal{M},v)\vDash [\mathcal{F}(\varphi_S)]\psi$. To that goal notice that since $(\mathcal{M},w)\vDash \bigwedge_{p\in S}\mathsf{h}_ap$, then by maximality of $e'$ with respect to $\varphi_S$ and product update definition, $\bigwedge_{p\in S}\mathsf{h}_ap\in e'$, and so by \textsc{Attentiveness} $\bigwedge_{p\in S} \ell(p)\in f$, for all $f\in Q'_a[e']$. So $f$ is indeed maximal with respect to $\varphi_S$ and thus $f\in E'_d$. 
			Hence, we have that for all $v\in R_a[w]$, $(\mathcal{M},v)\vDash [\mathcal{F}(\varphi_S)]\psi$, which by semantics of belief modality implies that $(\mathcal{M},w)\vDash B_a( [\mathcal{F}(\varphi_S)]\psi)$, as we wanted to conclude.
			
			($\Leftarrow$) For this other direction, the goal is showing that by assuming $(\mathcal{M},w)\vDash \varphi \rightarrow \bigvee_{S\subseteq At(\varphi)}(( \bigwedge_{p\in S}\mathsf{h}_ap \wedge \bigwedge_{p\in At(\varphi)\setminus S}\neg \mathsf{h}_ap) \rightarrow B_a( [\mathcal{F}(\varphi_S)]\psi))$ we can conclude that $(\mathcal{M},w)\vDash [\mathcal{F}(\varphi)]B_a\psi$. Here we proceed by contraposition and so show that by assuming $(\mathcal{M},w)\not\vDash [\mathcal{F}(\varphi)]B_a\psi$, i.e., by assuming that $\mathcal{F}(\varphi)$ is applicable in $(\mathcal{M},w)$ but $(\mathcal{M}^\varphi,(w,e))\not\vDash B_a\psi$, we can conclude that $(\mathcal{M},w)\not\vDash $ $\varphi \rightarrow \bigvee_{S\subseteq At(\varphi)}(( \bigwedge_{p\in S}\mathsf{h}_ap \wedge \bigwedge_{p\in At(\varphi)\setminus S}\neg \mathsf{h}_ap) \rightarrow B_a( [\mathcal{F}(\varphi_S)]\psi))$, i.e., we can conclude that if $(\mathcal{M},w)\vDash \varphi $ then $(\mathcal{M},w)\not\vDash \bigvee_{S\subseteq At(\varphi)}((\bigwedge_{p\in S}\mathsf{h}_ap \wedge \bigwedge_{p\in At(\varphi)\setminus S}\neg \mathsf{h}_ap) \rightarrow B_a( [\mathcal{F}(\varphi_S)]\psi))$, which means concluding that if $(\mathcal{M},w)\vDash \bigvee_{S\subseteq At(\varphi)}(\bigwedge_{p\in S}\mathsf{h}_ap \wedge \bigwedge_{p\in At(\varphi)\setminus S}\neg \mathsf{h}_ap)$ then $(\mathcal{M},w) \not\vDash  B_a([\mathcal{F}(\varphi_S)]\psi)$, which again means that there exists a $v\in R_a[w]$ with  $(\mathcal{M},v)\not\vDash [\mathcal{F}(\varphi_S)]\psi$, i.e., $\mathcal{F}(\varphi_S)$ is applicable in $(\mathcal{M},v)$  but $(\mathcal{M},v)\not\vDash \psi$. 
			Also here the conclusion will follow straightforwardly by using Lemma \ref{cor1}.
			
			So we start by making all the stated assumptions. Let $\mathcal{F}(\varphi)$ be applicable to $(\mathcal{M},w)$ and let $(\mathcal{M}^\varphi,(w,e))\not\vDash B_a\psi$, for some $a\in Ag$. Moreover, let $(\mathcal{M},w)\vDash \varphi$ and let $S\subseteq At(\varphi)$ be the unique $S$ such that $(\mathcal{M},w)\vDash \bigwedge_{p\in S}\mathsf{h}_ap \wedge \bigwedge_{p\in At(\varphi)\setminus S}\neg \mathsf{h}_ap$. The goal is to show that for this particular $S$, we also have $(\mathcal{M},w)\not\vDash B_a([\mathcal{F}(\varphi_S) ]\psi)$. 
			As by assumption the event model $\mathcal{F}(\varphi)$ is applicable in $(\mathcal{M},w)$, then also the event model $\mathcal{F}(\varphi_S)$ is applicable in $(\mathcal{M},w)$, and $(\mathcal{M}^{\varphi_S},(w,e'))$ exists. 
			
			Now $(\mathcal{M}^\varphi,(w,e))\not\vDash B_a\psi$ implies by semantics of belief modality that there exists some $(v,f)\in R^\varphi_a[(w,e)]$ such that $(\mathcal{M}^\varphi,(v,f))\not\vDash\psi$. As the assumptions of Lemma~\ref{cor1} are satisfied here, then $R_a^{\varphi}[(w,e)]= R_a^{\varphi_S}[(w,e')]$, and all the $(v,f)\in R_a^{\varphi}[(w,e)]$ are such that $\mathcal{M}^{\varphi_S},(v,f)\leftrightarroweq \mathcal{M}^{\varphi},(v,f)$. As modal equivalence follows by standard results on bisimulation and Kripke models (see e.g., \cite{blac.ea:moda}), then it follows that there exists some $(v,f)\in R^{\varphi_S}_a[(w,e')]$ such that $(\mathcal{M}^{\varphi_S},(v,f))\not\vDash\psi$. This means that there exists some $v\in R_a[w]$ and $f\in Q'_a[e']$ such that $(\mathcal{M}^{\varphi_S},(v,f))\not\vDash\psi$. As $(\mathcal{M},w)\vDash \bigwedge_{p\in S}\mathsf{h}_ap$, 
			then $\bigwedge_{p\in S}\mathsf{h}_ap\in e'$ and by \textsc{Attentiveness} $\bigwedge_{p\in S}(\mathsf{h}_ap\wedge \ell(p))\in f$ for all $f\in Q'_a[e']$. So $f$ is maximal with respect to $\varphi_S$ and thus $f\in E'_d$. It was necessary to show maximality of $f$ here as we now know that $\mathcal{F}(\varphi_S)$ is applicable in $(\mathcal{M},v)$ and so we know, by semantics of dynamic modality, that there exists some $v\in R_a[w]$ that is such that  $(\mathcal{M},v)\not\vDash [\mathcal{F}(\varphi_S)]\psi$. So we have that $(\mathcal{M},v)\vDash \bigwedge_{p\in S}\ell(p)$ and $(\mathcal{M},v)\not\vDash [\mathcal{F}(\varphi_S)]\psi$, that is $(\mathcal{M},v)\not\vDash \bigwedge_{p\in S}\ell(p)\rightarrow [\mathcal{F}(\varphi_S)]\psi$. Hence, by $v\in R_a[w]$, we can conclude that $(\mathcal{M},w)\not\vDash B_a[\mathcal{F}(\varphi_S)]\psi$.
	\end{proof}
	
	\section{Defaults}
	
	In event models for propositional attention, inattentive agents maintain their beliefs about what has been announced but they did not attend. Then, agents like Ann, who didn't hold any particular belief about the gorilla before watching the video and did not notice any while watching it, will not have any particular belief about it after having watched the video either. While this specific way of updating beliefs may be realistic and even rational in some cases, in many others, humans seem to update differently. As said in the introduction, in inattentional blindness situations agents that did not pay attention to an event and received no information about it often believe that the event did not happen. In these situations, agents seem to update their beliefs with respect to unattended events as well, regardless of whether their experience of the situation actually contained any evidence about them.

	In this section we propose to account for these specific belief updates by introducing \emph{default values}. A default value for an atom $q$ is either $q$, $\neg q$ or $\top$. If $q$ has default value $q$ for agent $a$ in a given announcement, it means that, in lack of evidence about $q$, agent $a$ will believe $q$ to be true. If $q$ means ``the basketball players are wearing shoes'', then an agent seeing the video might start to believe $q$ even without actually having paid attention to $q$, but just assuming $q$ to be true, as it would normally be true in such circumstances. 
	Similarly, if $q$ means ``a gorilla is passing by'', then agent $a$ might have $\neg q$ as the default value: if the occurrence of a gorilla is not paid attention to, the agent will believe there was none. Finally, if $q$ takes default value $\top$, it means that the agent doesn't default to any value, but preserves her previous beliefs. 
	Maybe she has no strong beliefs about whether all the basket ball players are wearing white, and hence if $q$ denotes that they are all wearing white, her default value for $q$ would be $\top$. We can think of default values as representing some kind of qualitative priors: They encode what an agent believes about what normally occurs in a given situation, and where those beliefs are sufficiently strong to let agent update her beliefs
	using these priors even when no direct evidence for or against them is observed (paid attention to).  
	\begin{definition}[Default Event Model $\mathcal{E}(\varphi, d)$]\label{event-default-varphi} Suppose $\varphi=\ell(p_1)\wedge \dots \wedge \ell(p_n)$, and suppose that $d$ is a \emph{default map}: 
		to each agent $a$ and atom $p_i$, $d$ assigns a \emph{default value} $d_a(p_i) \in \{ p_i, \neg p_i, \top \}$.
		The \emph{default event model} $\mathcal{E}(\varphi, d)=((E,Q,id_{E}), E_d)$ is: 
		\begin{multline*}
			E=\{\bigwedge_{p \in S} \ell(p)\ \wedge\!\!\!\bigwedge_{p \in \mathit{At}(\varphi)\setminus S} \!\!d_b(p)\ \wedge \bigwedge_{a \in Ag} \bigl(
			\bigwedge_{p \in X_a} \!\!\mathsf{h}_a p\ \wedge \bigwedge_{
				p \in S \setminus X_a} \!\!\neg \mathsf{h}_a p 
			\bigr) \colon \\b \in Ag, S \subseteq \mathit{At}(\varphi) \text{ and for all }a \in Ag, X_a \subseteq S \}
		\end{multline*}

		
		$Q_a$ is such that $(e,f)\in Q_a$ iff all the following hold for all $p$: 
		\begin{itemize}
			\item[-] 
			\textsc{Attentiveness}: if $\mathsf{h}_a p\!\in\! e$ then $\mathsf{h}_a p,\ell(p)\!\in\! f$;
			\item[-] \textsc{Defaulting}: if $\mathsf{h}_a p \notin e$ then $d_a(p)\in f.$
		\end{itemize} 
		
		$E_d=\{\psi\in E\colon \ell(p)\in \psi, \text{ for all } \ell(p)\in \varphi\}$. 
		
	\end{definition}
	
	Default event models differ from event models for propositional attention in that if an event in a default model does not contain a literal from the announced formula, then it contains its default value for one of the agents. Each event contains default values for one agent only, so that no event may contain contradicting default values. The accessibility relations are given by similar principles as above, with the difference that the second principle is now called \textsc{Defaulting}, and this principle implies that inattentive agents only consider possible the default values of what they left unattended. Note that defaults are common knowledge among the agents (the event model doesn't encode any uncertainty about the default map $d$). 
Figure~\ref{figure:update} (right) illustrates the revised update of our initial model with the default event model representing Ann seeing the video.
In lack of attention to $g$, she defaults to $\neg g$, the intuition being that she believes that she would see the gorilla had it been there. She comes to believe there is no gorilla: $(\mathcal{M}'',w'')\vDash B_a \neg g$.

\paragraph{Axiomatization} 

The axiomatization of the logic for propositional attention with defaults is given by the same axioms as in Table~\ref{tab:logic}, except  for the axiom for belief dynamics which is replaced by the following axiom where inattentive agents adopt the default option for the unattended atoms (where $\varphi = \ell(p_1)\wedge \dots \wedge \ell(p_n)$). For $\varphi_{Sd}=\bigwedge_{p\in S} \ell(p)\wedge \bigwedge_{p\in At(\varphi)\setminus S} d_a(p)$, call the resulting table \emph{Table~2}:

\begin{multline*}
[\mathcal{E}(\varphi, d)]B_a\psi \leftrightarrow (\varphi \rightarrow \bigvee_{S\subseteq At(\varphi)}\bigl( (\bigwedge_{p\in S} \mathsf{h}_ap\wedge \bigwedge_{p\in At(\varphi)\setminus S}\neg \mathsf{h}_ap)\\\rightarrow B_a( [\mathcal{E}(\varphi_{Sd},d) ]\psi))\bigr)
\end{multline*}

To prove 
soundness and completeness,
we need a lemma similar to Lemma~\ref{cor1}.

\begin{lemma} \label{lemma3}For any pointed Kripke model $(\mathcal{M},w)$ with $(\mathcal{M},w)\vDash \varphi$, and for any $a\in Ag$, consider the $S\subseteq At(\varphi)$ that is such that $(\mathcal{M},w)\vDash \bigwedge_{p\in S} \mathsf{h}_ap \wedge \bigwedge_{p\in At(\varphi)\setminus S}\neg \mathsf{h}_ap$. Let $\varphi_{Sd} = \bigwedge_{p\in S}\ell(p)\wedge\bigwedge_{p\in At(\varphi)\setminus S}d_a(p)$. The updated models $(\mathcal{M},w)\otimes \mathcal{E}(\varphi_{Sd}, d)=((W^{\varphi_{Sd}}, R^{\varphi_{Sd}}, V^{\varphi_{Sd}}),(w,e'))$ and $(\mathcal{M},w)\otimes \mathcal{E}(\varphi,d)=((W^{\varphi}, R^{\varphi}, V^{\varphi}),(w,e))$ are such that \begin{enumerate}
		\item $R_a^{\varphi}[(w,e)]= R_a^{\varphi_{Sd}}[(w,e')]$
		\item For all $(v,f)\in R_a^{\varphi}[(w,e)]$, $(\mathcal{M}^\varphi,(v,f))\leftrightarroweq(\mathcal{M}^{\varphi_{Sd}},(v,f))$.
	\end{enumerate}
\end{lemma}
\begin{proof} The proofs of both (1) and (2) proceed analogously to the proofs of (1) and (2) of Lemma \ref{cor1}, respectively. We hence only show left to right of (1). We follow similar notational conventions as in Lemma~\ref{cor1}, letting $\mathcal{E}(\varphi,d)=((E,Q,pre),E_d)$ and $\mathcal{E}(\varphi_{Sd},d)=((E',Q',pre'),E'_d)$.
	
	Let $(v,f)\in R^{\varphi}[(w,e)]$. This means that $v\in R_a[w]$ and $f\in Q_a[e]$.  Then, to reach the desired result that $(v,f)\in R^{\varphi_{Sd}}[(w,e')]$, we only need to show that $f\in Q'_a[e']$, as then we would have that $v\in R_a[w]$ and $f\in Q'_a[e']$, and since $(\mathcal{M},v)\vDash pre(f)$ then $(v,f)\in W^{\varphi_S}$ and we could conclude that $(v,f)\in R^{\varphi_{Sd}}[(w,e')]$. Similarly to the proof above, we show this by showing that $f$ is such that $f\in E'$ and that $f$ satisfies the requirements that \textsc{Attentiveness} and \textsc{Defaulting} pose to belong to $Q'_a[e']$, i.e., it contains the needed formulas.

	So let's first see what formulas $f$ contains. By initial assumption, $(\mathcal{M},w)\vDash (\bigwedge_{p\in S} \mathsf{h}_ap \wedge \bigwedge_{p\in At(\varphi)\setminus S}\neg \mathsf{h}_ap)$ for some $S\subseteq At(\varphi)$. As $(w,e)\in W^\varphi$, then by product update definition and maximality of $e$, it holds that $(\bigwedge_{p\in S} \mathsf{h}_ap \wedge \bigwedge_{p\in At(\varphi)\setminus S}\neg \mathsf{h}_ap)\in e$. Then by \textsc{Attentiveness} and $\bigwedge_{p\in S} \mathsf{h}_ap\in e$, we know that $\bigwedge_{p\in S} (\ell(p)\wedge \mathsf{h}_ap)\in f$. Moreover, as $\bigwedge_{p\in At(\varphi)\setminus S} \neg \mathsf{h}_ap\in e$, then by def. of event model for propositional attention with defaults (in particular by definition of its set of events) for all $p\in At(\varphi)\setminus S$, $\mathsf{h}_ap\notin e$ and so by \textsc{Defaulting}, $f$ contains $\bigwedge_{p\in At(\varphi)\setminus S}d_a(p)$, which then implies that $\mathsf{h}_ap\notin f$ for all such $p\in At(\varphi)\setminus S$, by def. of event models for propositional attention with defaults. Hence, $f$ is such that $\bigwedge_{p\in S} (\ell(p)\wedge \mathsf{h}_ap) \land \bigwedge_{p\in At(\varphi)\setminus S}d_a(p)\in f$ and, for all $p\in At(\varphi)\setminus S$, $\mathsf{h}_ap\notin f$. 

	Now let's see what is required to belong to $Q'_a[e']$. Since by initial assumption $(\mathcal{M},w)\vDash \bigwedge_{p\in S} \mathsf{h}_ap$ and since $(w,e')\in W^{\varphi_{Sd}}$, then by product update definition and maximality of $e'$, it holds that $\bigwedge_{p\in S} \mathsf{h}_ap \in e'$. Then we can use \textsc{Attentiveness} to see that in order to belong to $Q'_a[e']$ an event must contain $\bigwedge_{p\in S} (\ell(p)\wedge \mathsf{h}_ap)$. Moreover, as $\bigwedge_{p\in At(\varphi)\setminus S} \neg \mathsf{h}_ap\in e$, then by \textsc{Defaulting}, all events in $Q'_a[e']$ must contain $\bigwedge_{p\in At(\varphi)}d_a(p)$ which implies, by the way events with defaults are defined, that they must not contain $\mathsf{h}_ap$ for all such $p\in At(\varphi)\setminus S$. Hence, to belong to $Q'_a[e']$, an event $f'$ must be such $\bigwedge_{p\in S} (\ell(p)\wedge \mathsf{h}_ap)\wedge \bigwedge_{p\in At(\varphi)\setminus S} d_a(p) \in f'$ as well as, for all $p\in At(\varphi)\setminus S, \mathsf{h}_ap\notin f'$. As this is exactly what we have with $f$, then $f\in E'$ and $f\in Q'_a[e']$.

	Hence, we have that if $f\in Q_a[e]$ then $f\in Q'_a[e']$. Above we assumed that $(v,f)\in R^{\varphi}[(w,e)]$, i.e., that $v\in R_a[w]$ and $f\in Q_a[e]$. This now implies that $v\in R_a[w]$ and $f\in Q'_a[e']$, and by def. of product update that $(v,f)\in R_a^{\varphi_{Sd}}[(w,e')]$, which is what we wanted to conclude.
\end{proof}

Recalling that we call Table 2 the table resulting from replacing the axiom for the belief dynamics in Table 1 with the new axiom for the defaults introduced in the beginning of this section, we now have the following.
\begin{theorem} \label{default sound and complete} 
	The axiomatization in Tbl. 2 is sound and complete.
\end{theorem}	
\begin{proof}

	\noindent\textit{Completeness:} It proceeds by usual reduction arguments \cite{ditmarsch2007dynamic}. 
	\textit{Soundness:} We show that axioms and inferences rules from Table~2 are valid. Using the same reasoning as in the previous soundness proof, we only show here the validity of the reduction axiom for attention-based belief updates with defaults. We prove the two directions separately.
	
	Let $(\mathcal{M},w)=((W,R,V),w)$ be a pointed Kripke model. We will use $\varphi_{Sd}$ in the same way as above, and we will use also the same notation for $\mathcal{E}(\varphi,d)$ and $\mathcal{E}(\varphi_{Sd},d)$, as well as for $(\mathcal{M}^\varphi,(w,e))$ and $(\mathcal{M}^{\varphi_{Sd}},(w,e))$.
	
	($\Rightarrow$) 	We want to prove that if we assume $(\mathcal{M},w)\vDash 	[\mathcal{E}(\varphi, d)]B_a\psi$, $(\mathcal{M},w)\vDash \varphi $ and  $(\mathcal{M},w)\vDash \bigwedge_{p\in S} \mathsf{h}_ap\wedge \bigwedge_{p\in At(\varphi)\setminus S}\neg \mathsf{h}_ap$, then it follows that $(\mathcal{M},w)\vDash B_a[\mathcal{E}(\varphi_{Sd},d) ]\psi$. The proof strategy is analogous to the strategy of the previous soundness proof in the left to right direction. 
	
	So assume $(\mathcal{M},w)\vDash 	[\mathcal{E}(\varphi, d)]B_a\psi$ and $(\mathcal{M},w)\vDash \varphi $ and consider the unique $S\subseteq At(\varphi)$ that is such that $(\mathcal{M},w)\vDash \bigwedge_{p\in S} \mathsf{h}_ap\wedge \bigwedge_{p\in At(\varphi)\setminus S}\neg \mathsf{h}_ap$. As $(\mathcal{M},w)\vDash \varphi$ then $(\mathcal{M}^\varphi,(w,e))$ exists. As $(\mathcal{M},w)\vDash [\mathcal{E}(\varphi,d)]B_a\psi$ then by semantics of the dynamic modality and by applicability of $\mathcal{E}(\varphi,d)$ to $(M,w)$, $(\mathcal{M}^\varphi,(w,e))\vDash B_a\psi$, which implies, by semantics of belief modality, that for all $(v,f)\in R^\varphi_a[(w,e)], (\mathcal{M}^\varphi,(v,f))\vDash \psi$. By Lemma~\ref{lemma3}, we know that $R^\varphi_a[(w,e)]=R^{\varphi_{Sd}}_a[(w,e')]$ and that for all $(v,f)\in R^\varphi_a[(w,e)], (\mathcal{M}^\varphi, (v,f))\leftrightarroweq (\mathcal{M}^{\varphi_{Sd}}, (v,f))$. By standard modal logic results, bisimulation implies modal equivalence, and so it follows that also for all $(v,f)\in R^{\varphi_{Sd}}_a[(w,e')]$, it is the case that  $(\mathcal{M}^{\varphi_{Sd}},(v,f))\vDash \psi$, which is equivalent to saying that for all $v\in R_a[w]$ and for all $f\in Q'_a[e']$ that are such that $(v,f)\in W^{\varphi_{Sd}}$, $(\mathcal{M}^{\varphi_{Sd}},(v,f))\vDash \psi$.
	
	Now as in the previous soundness proof we have two cases: either $\mathcal{E}(\varphi_{Sd},d)$ is applicable to $(\mathcal{M},v)$ or it is not. If it is not, then $(\mathcal{M},v)\vDash [\mathcal{E}(\varphi_{Sd},d)]\psi$. If instead $\mathcal{E}(\varphi_{Sd},d)$ is applicable to $(\mathcal{M},v)$ we need to show maximality of $f$ for all such $f\in Q_a[e']$, i.e., $f\in E'_d$, to then infer by semantics of the dynamic modality, that $(\mathcal{M},v)\vDash [\mathcal{E}(\varphi_{Sd},d)]\psi$ for all $v\in R_a[w]$. The argument proceed similarly to the previous proof, namely, since $(\mathcal{M},w)\vDash \bigwedge_{p\in S} \mathsf{h}_ap\wedge \bigwedge_{p\in At(\varphi)\setminus S} \neg \mathsf{h}_ap$, then $\bigwedge_{p\in S}\mathsf{h}_ap\wedge \bigwedge_{p\in At(\varphi)\setminus S} \neg \mathsf{h}_ap\in e'$. By $\bigwedge_{p\in S}\mathsf{h}_ap\in e'$ we know that by \textsc{Attentiveness}, $\bigwedge_{p\in S}\ell(p)\in f$, and by $\bigwedge_{p\in At(\varphi)\setminus S} \neg \mathsf{h}_ap\in e'$ we know that by \textsc{Defaulting}, $\bigwedge_{p\in At(\varphi)\setminus S} d_a(p)\in f$, for all $f\in Q_a[e']$. Hence, $\bigwedge_{p\in S}\ell(p)\wedge \bigwedge_{p\in At(\varphi)\setminus S} d_a(p)\in f$, for all $f\in Q_a[e']$. This means that all such $f$ are indeed maximal with respect to $\varphi_{Sd}$ and so $f\in E'_d$. Hence, by semantics of dynamic modality, we now get that for all $v\in R_a[w]$, $(\mathcal{M},v)\vDash  [\mathcal{E}(\varphi_{Sd},d)]\psi$, and thus also that $(\mathcal{M},w)\vDash B_a ([\mathcal{E}(\varphi_{Sd},d)]\psi)$, as we wanted to conclude.
	
	($\Leftarrow$) The right to left direction proceeds similar to the right to left in the proof of Theorem 4.2, i.e., by using contraposition and Lemma~\ref{lemma3} we can conclude the desired result.
\end{proof}

\begin{example}\label{example:doctor}
In the introduction, we mentioned the potential application of our models for human-robot collaboration. Consider an emergency scenario with a mixed human-robot rescue team including a human doctor $a$ and an assisting robot $b$. Suppose $a$ is attending to an injured victim and that $b$ is ready to assist. While she is attending to the victim, fire breaks out and creates a dangerous situation. The doctor, being absorbed in trying to help the victim, has not noticed the fire, and so it makes sense for the robot to inform her. This scenario is completely equivalent to the invisible gorilla example with $p$ instead meaning, say, ``the victim is injured'' and
$g$ meaning ``fire has broken out''. The point is that the after fire has broken out, we are in the situation of Figure~\ref{figure:update} (right) where $g \land 
B_b B_a \neg g$ holds: The robot correctly believes that the doctor has a false belief that there is no fire. A proactive robot should inform its human team members about any false beliefs that could lead to catastrophic outcomes. This requires the ability of the robot to model those false beliefs, including false beliefs arising due to inattentional blindness, which is exactly what our models provide.   

\end{example}

\section{Syntactic event models} 
The event models introduced above are rather large. The event models for propositional attention grow exponentially with the number of agents: For each subset of agents $A \subseteq Ag$ and each announced atom $p$, it contains at least one event where all $\mathsf{h}_a p$, $a \in A$ occur positively, and all $\mathsf{h}_a p$, $a \in Ag \setminus A$ occur negatively. They also grow exponentially in the number of announced atoms: For each subset $S$ of atoms in the announced formula $\varphi$, it contains at least one event in which the set of propositional atoms occurring is exactly $S$. 
However, note that we still managed to represent the event models
in a relatively compact way in terms of a set of precondition formulas and a list of simple edge principles. This leads us to the following questions. Can we represent \emph{any} event model---or at least a sufficiently general subclass of them---in terms of a set of precondition formulas and a set of
edge principles? 
If so, can we then use this to define syntactically represented event models where the edges are defined by formulas representing the edge principles? This would give us a formally more precise way of handling principle-based event models. Would that then lead to more succinctly represented event models? 

We are not the first to consider ways to represent event models succinctly and syntactically. Aucher~\cite{aucher2012sequents} defined a language with special atoms $p'_\varphi$ meaning ``$\varphi$ is the precondition of the current event''. However, to be able to represent our edge principles via formulas, we need to be able to reason about the structure of the event preconditions, for instance when we want to say that some literal is contained in a precondition (like $\mathsf{h}_a p \in e$). Therefore it doesn't suffer for our purposes to introduce formulas where the preconditions are treated as atomic entities. Another approach is by Charrier and Schwarzentruber~\cite{charrier2017succinct}. In their language, it is possible to reason about the precondition formulas, for instance the formula $(p_e \to p) \land (p_f \to \top)$ can be used to express that event $e$ has precondition $p$ and event $f$ has precondition $\top$. They then represent edges by a program in PDL (propositional dynamic logic). This gives a very imperative representation of the edges, whereas we are here looking for a more declarative representation matching the edge principles introduced above.  

We now introduce a new formal language to be used to describe event models. 
Where $\psi \in \mathcal{L}$, the \emph{event language} $\mathcal{L_E}$ is: 
\begin{align*}
\varphi &::= \psi\!\Rightarrow\!\mathsf{e}  \mid \mathsf{e}\!\Rightarrow\!\psi \mid \neg \varphi \mid \varphi \vee \varphi \mid \Box \varphi 
\end{align*}
The formula $\psi\!\Rightarrow\!\mathsf{e}$ is read as ``$\psi$ implies the precondition of the (current) event'' and $\mathsf{e}\!\Rightarrow\!\psi$ as ``the precondition of the (current) event implies $\psi$''. We will use $e\!\Leftrightarrow\!\psi$ as shorthand for $\psi\!\Rightarrow\!\mathsf{e} \land \mathsf{e}\!\Rightarrow\!\psi$. Formulas of $\mathcal{L_E}$ are to be evaluated in single-agent event models, since we are going to specify the edge principles for each agent $a$ by a separate formula $\varphi_a$ of $\mathcal{L_E}$.
\begin{definition}[Satisfaction] 
Let $\mathcal{E} = (E,Q,pre)$
be a single-agent event model over $\mathcal{L}$ (so $Q \subseteq E^2$).  For any $e \in E$, satisfaction of $\mathcal{L_E}$-formulas in $\mathcal{E}$ is given by the following clauses extended with the standard clauses for the propositional connectives:
\noindent \begin{center}
\begin{tabular}{lll}
$(\mathcal{E},e) \vDash   \psi\!\Rightarrow\!\mathsf{e}$ & iff & $\vDash  \psi \to pre(e)$;\tabularnewline 
$(\mathcal{E},e) \vDash  \mathsf{e}\!\Rightarrow\!\psi $ & iff & $\vDash  pre(e) \to \psi$;\tabularnewline
$(\mathcal{E},e) \vDash \Box \varphi $ & iff & $(\mathcal{E},f) \vDash \varphi$ for all $(e,f)\in Q$.\tabularnewline

\end{tabular}
\par\end{center}			
A formula $\psi$ is called \emph{valid} in $\mathcal{E} = (E,Q,pre)$ if $(\mathcal{E}, e) \vDash \psi$ holds for all $e \in E$. We then write $\mathcal{E} \vDash \psi$. To have a convient notation for reasoning about what holds true for a single event with precondition $\varphi \in \mathcal{L}$, we introduce the following notation, where $\psi \in \mathcal{L_E}$:
\noindent \begin{center}
\begin{tabular}{lll}
$\varphi \vDash \psi$  & iff & $((\{\varphi\}, \emptyset, id_{\{\varphi\}}),\varphi) \vDash  \psi $ 
\end{tabular}\par\end{center}	
\end{definition}
Note that the $\mathsf{e}$ in the syntax is bound to the event $e$ at which the formula is evaluated. So $\mathsf{e}\!\Rightarrow\!p\to \Box \mathsf{e}\!\Rightarrow\!\neg p$ means that if the precondition of the current event implies $p$, then the precondition of any accessible event implies $\neg p$.
Concerning the notation $\varphi \vDash \psi$, note that we for instance have
$p \land q \vDash \mathsf{e}\!\Rightarrow\! p \land  \mathsf{e}\!\Rightarrow\! q$: Both $p$ and $q$ are implied by an event with precondition $p \land q$.
Note that the $\Rightarrow\!\mathsf{e}$ operator is not truth-functional: For instance we have $\top \vDash \mathsf{e}\!\Rightarrow\!(p \vee \neg p)$, but we don't have $\top \vDash \mathsf{e}\!\Rightarrow\!p  \vee \mathsf{e}\!\Rightarrow\!\neg p$.


\begin{example}
Consider the event model $\mathcal{E}'(\varphi)$ of Definition~\ref{a-star-varphi} for some $\varphi \in \mathcal{L}$ where $Ag = \{a\}$. By \textsc{Inertia}, if $e$ is an event not containing $\mathsf{h}_a$, then for any other event $f$ with $(e,f) \in Q_a$, we have $f = \top$. We can express this using an $\mathcal{L_E}$-formula: $\neg \mathsf{e}\!\Rightarrow\!\mathsf{h}_a \to \Box \mathsf{e}\!\Leftrightarrow\!\top$. The formula says: if $\mathsf{h}_a$ is not implied by the precondition of the current event, then any accessible event has a precondition equivalent to $\top$. The formula is simply \textsc{Inertia} expressed in $\mathcal{L_E}$, and we have $\mathcal{E}'(\varphi) \vDash \neg \mathsf{e}\!\Rightarrow\!\mathsf{h}_a \to \Box \mathsf{e}\!\Leftrightarrow\!\top$. 
\end{example}
When trying to come up with a new way of representing event models syntactically, there is a trade-off between generality and expressivity on one side and succinctness and elegance on the other. The more general a class of event models we want to be able to describe, the more complex the language might have to be and the longer and more complicated the formulas might become. Here we will aim for keeping things simple, even if it implies less generality. For instance, opposite the approach of \cite{charrier2017succinct}, we decided not to include propositional atoms in $\mathcal{L_E}$ for referring to the names of specific events. This limits  expressivity, as then the language can only distinguish events by their preconditions and can not represent distinct events with the same precondition.
However, for the event models of this paper, this is not a limitation. 

We move to define our syntactic event models. To make the distinction clear, we will now refer to the standard event models of Definition~\ref{def: event model} as \emph{semantic event models}.
\begin{definition}
A \emph{syntactic event model} is a pair $\mathcal{G} = (\psi_E,(\psi_a)_{a\in Ag})$, where all the $\psi$ formulas belong to $\mathcal{L_E}$. 
The semantic event model $\mathcal{H} = (E,Q,id_E)$ \emph{induced} by $\mathcal{G}$ is defined as follows:
\begin{itemize}
\item[-]  $E = \{ \varphi \in \mathcal{L}: \varphi$ is a conjunction of literals s.t.\ $ \varphi \vDash \psi_E \}$;
\item[-] For all $a \in Ag$, $Q_a$ is the largest subset of $E^2$ satisfying $(E,Q_a,id_E) \vDash \psi_a$. If such a unique largest set doesn't exist, let $Q_a$ be the empty set. 
\end{itemize}   
Where $\psi_{E_d} \in \mathcal{L_E}$, we call $(\mathcal{G},\psi_{E_d})$ a \emph{syntactic multi-pointed event model}. The \emph{induced} multi-pointed event model of $(\mathcal{G},\psi_{E_d})$ is $(\mathcal{H},E_d)$ where $\mathcal{H}$ is the event model induced by $\mathcal{G}$ and $E_d = \{ \varphi \in \mathcal{L} : \varphi$ is a conjunction of literals s.t.\ $\varphi \vDash \psi_{E_d} \}$.
\end{definition}

\begin{example}\label{example:first syntactic}
Consider again the event model $\mathcal{E}'(\varphi)$ of Def.~\ref{a-star-varphi}, where we here let $\varphi = q$, assume $At = \{q\}$ and assume $Ag$ to be any set of agents.
Then $\mathcal{E}'(\varphi)$ is induced by the syntactic event model $\mathcal{G} = (\psi_E, (\psi_a)_{a \in Ag})$ defined as follows:

$\psi_E = \mathsf{e}\!\Leftrightarrow\!\top \vee \bigl((
\mathsf{e}\!\Rightarrow\! q \lor  \mathsf{e}\!\Rightarrow\!\neg q)\ \land \bigwedge_{a \in Ag} ((\mathsf{e}\!\Rightarrow\!\mathsf{h}_a) \vee (\mathsf{e}\!\Rightarrow\!\neg \mathsf{h}_a))$

$\psi_a = (\mathsf{e}\!\Rightarrow\!\mathsf{h}_a \to \Box \mathsf{e}\!\Rightarrow\!q) \land (\neg \mathsf{e}\!\Rightarrow\!\mathsf{h}_a \to \Box \mathsf{e}\!\Leftrightarrow\!\top)$.

\noindent 
The definition of $\psi_E$ states that any event is either (equivalent to) $\top$ or else: 1) it implies either $q$ or $\neg q$ and, 2) for all $a \in Ag$, it implies either $\mathsf{h}_a$ or $\neg \mathsf{h}_a$. Note that since the induced event model is always a model over a set of conjunctive preconditions, we can reformulate this as follows: $\psi_E$ states that any event is either $\top$ or else 1) it contains either $q$ or $\neg q$ and, 2) for all $a \in Ag$, it contains either $\mathsf{h}_a$ or $\neg \mathsf{h}_a$. Comparing with Definition~\ref{a-star-varphi}, we see that this is exactly how we defined the set of events of this model.   
Concerning $\psi_a$, we earlier concluded that the second conjunct expresses \textsc{Inertia}. The first conjunct expresses \textsc{Basic Attentiveness}. 
\end{example}

The \emph{size} of a syntactic event model is the sum of the lengths of the formulas it consists of. We say that two semantic event models $\mathcal{E} = (E,Q,pre)$ and $\mathcal{E}' = (E',Q',pre')$ are \emph{equivalent} if there exists $e \in E, e' \in E'$ such that for all pointed Kripke models $\mathcal{M} = ((W,V,R),w)$ and all formulas $\varphi \in \mathcal{L}$, $\mathcal{M} \otimes \mathcal{E}, (w,e) \vDash \varphi$ iff $\mathcal{M} \otimes \mathcal{E}', (w,e') \vDash \varphi$.\footnote{We could have defined this notion equivalently in terms of bisimulations~\cite{charrier2017succinct}, but as we haven't defined bisimulations in this paper, we choose this equivalent formulation~\cite{kooi2011arrow}.} We can now prove an exponential succinctness result for syntactic event models. We show that for all $n\geq 1$, we can construct a particular syntactic event model $\mathcal{G}(n)$ that can't be represented by any semantic event model with less than $2^n$ events. 
\begin{proposition}[Exponential succinctness]
There exists syntactic event models  $\mathcal{G}(n)$, $n \geq 1$, such that all of the following holds: 
\begin{itemize}
\item $\mathcal{G}(n)$ has size $O(n)$.
\item The semantic event model  $\mathcal{H}(n)$ induced by  $\mathcal{G}(n)$ has $2^n$ events (and is hence of size $\Omega(2^n)$).
\item Any other semantic event model that is equivalent to $\mathcal{H}(n)$ will have at least $2^n$ events.    
\end{itemize}
Furthermore, we can construct the $\mathcal{G}(n)$ so that they use only one agent and where $n$ is the number of atomic propositions. 
\end{proposition}  
\begin{proof}
	For each $n \geq 1$, let $\mathcal{G}(n)$ denote the syntactic event model of Example~6.4 with $Ag = \{1, \dots, n \}$. Let $\mathcal{H}(n)$ denote the semantic event model induced by $\mathcal{G}(n)$. The induced event model $\mathcal{H}(n)$ is the one defined in Definition~3.2 
	that we already
	concluded to have at least $2^n$ events (due to there being one event per subset of $\{ \mathsf{h}_a: a \in Ag\}$, the subset containing the $\mathsf{h}_a$ that occur positively in the event precondition). In~\cite{charrier2017succinct}, it is proven that $\mathcal{H}(n)$ is not equivalent to a semantic event model with less than $2^n$ events. However, $\mathcal{G}(n)$ is of size $O(n)$, as we will now see. The formula $\psi_E$ is of size $O(n)$: the inner-most disjunction is repeated once for each agent, but everything else is of fixed size. The formula $\psi_a$ is also of fixed size, it simply has size (length) 22. We however need one of these formulas for each agent, so in total $(\psi_a)_{a \in Ag}$ also has size $O(n)$.
	This proves the required, except for the last point about using only one agent. To use only one agent, we need to turn to a different succinctness result, the one about arrow updates in~\cite{kooi2011arrow}. For all $n \geq 1$, let $\mathcal{L}(n)$ be the language with atomic propositions $P = \{1,\dots,n\}$ and a single agent $a$. For each $n \geq 1$, let $\mathcal{H}'(n)$ be the semantic event model over $\mathcal{L}(n)$ in which each subset of $P$ is an event, and there is an $a$-edge from event $P' \subseteq P$ to event $P'' \subseteq P$ if for some $i$, $p_i \in P'' \setminus P'$. The event model $\mathcal{H}'(n)$ clearly has $2^n$ events. In \cite{kooi2011arrow}, it is shown that  there exists no semantic event model equivalent to $\mathcal{H}'(n)$ having less than $2^n$ events. To complete our proof, we then only need to show that we can represent $\mathcal{H}'(n)$ using a syntactic event model of size $O(n)$. Let $\mathcal{G}'(n) = (\psi'_E,\psi'_a)$ be the syntactic event model over $\mathcal{L}(n)$ defined by $\psi'_E = \top$ and $\psi'_a = \bigvee_{ 1 \leq i \leq n } (\neg \mathsf{e}\!\Rightarrow\!p_i \rightarrow \Box  \mathsf{e}\!\Rightarrow\!p_i)$. It is simple to check that the semantic event model induced by $\mathcal{G}'(n)$ is exactly $\mathcal{H}'(n)$. Also, clearly $\mathcal{G}'(n)$ has size $O(n)$.  
\end{proof}
In the proof above we refer to a result~\cite{charrier2017succinct} showing that  their succinct event models are exponentially more succinct than semantic event models. However, their representation of the event models of Definition~\ref{a-star-varphi} are of size $O(n^2)$ with $n$ being the number of agents, whereas our syntactic event model are of size $O(n)$, hence even more compact (their PDL program for each agent has length $O(n)$, whereas our corresponding formula $\psi_a$ is of constant length). 

Note that the semantic event model induced by a syntactic event model is of a particular form where the event preconditions are conjunctions of literals. While this limits the kind of semantic event models we can represent using syntactic event models, it still covers a fairly generous class. Any semantic event model $\mathcal{H} = (E,Q,pre)$ where events are conjunctions of literals (as all event models of this paper are) can be turned it into a syntactic event model $\mathcal{G} = (\psi_E,(\psi_a)_{a \in Ag})$ by simply letting $\psi_E = \bigvee_{e \in E} \mathsf{e}\!\Leftrightarrow\!e$ and $\psi_a = \bigwedge_{e \in E} \bigl(\mathsf{e} \Leftrightarrow e \to \Box \bigvee_{(e,f) \in Q_a} \mathsf{e} \Leftrightarrow f \bigr)$.

We are here only using syntactic event models to provide simple and succinct representations of our semantic event models (that are otherwise of exponential size). However, it is relevant to mention that these new syntactic event models could potentially also be interesting for other reasons. As mentioned in~\cite{baltag2022logics}, there has been quite a lot of resistance to DEL based on (semantic) event models, since one is ``mixing syntax and semantics'' (due to event models being semantic objects, but still appearing inside modal operators in the language). A syntactic event model clearly does not have this problem, as it's a representation using a sequence of formulas from the language $\mathcal{L_E}$. All formulas are from the same language, so by slightly extending it, we could even represent an event model syntactically by a single formula of such an extended language. 

\section{Related and Future Work}
This work has built on a previous model for attentive agents~\cite{bolander2015announcements}, generalising the framework to model
(1) agents who may pay attention to strict subsets of propositions; (2) agents who may default to specific truth values for the atomic formulas they failed to attend. 

What we here call \emph{attention } is similar to what has been called \emph{observability} in the AI and DEL literature. Observability can be attached to different aspects of the world: to propositional atoms \cite{brenner2009continual,hoek2011knowledge}, to actions \cite{bolander2015announcements}, to actions of agents \cite{bolander2018seeing}, or to particular actions \cite{baral2012action}, and the same holds for attention. However, conceptually, attention and observability are not exactly the same, and in this paper we have been focusing on representing attention to propositional atoms mainly as a starting point for a richer model of attention.  

In addition, we proposed a syntactic description of event models that, besides working towards settling the mixture of syntax and semantics typical of DEL, allowed us to reach an exponential succinctness result. There is a clear relation to generalized arrow updates~\cite{kooi2011generalized}, but we conjuncture that our syntactic event models can be even more succinct than generalized arrow updates. 
We however leave this for future work.

The examples provided in this paper are arguably toy examples in the sense of involving few agents (2) and few propositional atoms (also 2). Since the semantic event models grow exponentially in both the number of agents and propositional atoms, the semantic representation doesn't scale well. However, the syntactic representation does, and in future work we'd like to consider whether we can define a product update directly in terms of syntactic event models to allow for better scalability of our framework. 

We also plan to extend the model further to include more core features of attention, for instance an upper bound on the number of atomic formulas that an agent can pay attention to (a bound on the \emph{attention capacity}). 
This is a simple tweak of the model, but it allows us to capture a lot more: attention as a bounded resource. This can be applied in at least two distinct ways. 
For instance in Example~\ref{example:doctor}, whether the doctor pays attention to the fire breaking out or not might obviously depend on how busy she is attending to other things. This is similar to the invisible gorilla, where attending to the ball passes seems to consume all of the attention capacity. Adding attention capacities would  allow the robot to have a more realistic model of human attention and when to intervene. The second use of attention capacities could be to apply it to allow robots to manage their own attention in order to save computational resources.
In the DEL literature, attention as a cognitive resource has been explored by~\cite{belardinelli2021bounded}, where an attention budget and a subjective cost for formulas to be learnt are introduced in the model. 

Attention may also relate to the notion of awareness \cite{schipper2014awareness}, as both concepts can be thought as imposing some limitation on the set of propositions the agent entertains. However, the two also differ: Awareness seems to be more about the propositions that the agent can conceive and thus uses to reason, whereas attention (at least for how we formalised it) is a restriction on what agents perceive of an announcement. In this sense, attention seems to be more about learning dynamics, whereas awareness less so. Future work will explore their relationship.



\balance

\begin{acks}
We gratefully acknowledge Rasmus K. Rendsvig for important inputs in the early phases of this project. We also gratefully acknowledge funding support by the Carlsberg Foundation through The 
Center for Information and Bubble Studies (CIBS). Finally, we thank the anonymous reviewers for helpful comments and feedback. 		
\end{acks}

\bibliographystyle{ACM-Reference-Format} 
\bibliography{litt_std}

\end{document}